\documentclass[fontsize=12pt]{scrarticle}

\usepackage{preamble}
\usepackage{preamble_math}
\usepackage{biblatex_setup}

\begin{document}

\maketitle
\begin{longtable}[c]{p{0.9\textwidth}}
\textbf{\small{}Abstract. }{\small{}Physical AI is being successfully applied to data which does not follow the traditional paradigm of \emph{independent and identically distributed} (i.i.d.) samples. In fact, physical AI is often trained on data which is not random at all, and is instead derived from chaotic dynamical systems like turbulence.

We aim to explain the empirical success of these methods using the example of \emph{generative adversarial networks} (GANs), whose statistical learning theory under the i.i.d. assumption is generally well understood. We prove that it is possible, using an infinite-dimensional model of \emph{generative adversarial learning} (GAL), to learn the invariant distribution of a sufficiently chaotic dynamical system from a single deterministically evolving time series of its states or measurements thereof, and give explicit rates for the convergence to the solution in terms of the Jensen-Shannon divergence.}{\small\par}

\smallskip{}

\textbf{\small{}Keywords:}{\small{} Ergodic Dynamical Systems $\bullet$ Generative Adversarial Learning
$\bullet$ Statistical
Learning Theory $\bullet$ Concentration Inequalities}


\end{longtable}

\smallskip{}

\section{Introduction}
Physical artificial intelligence is a new major trend in machine learning. In particular, many authors propose to replace traditional computational fluid dynamics (CFD) solvers by neural-network-based learning methods. Prominent examples are physics informed neural networks (PINN) \cite{raissi2019physics}, operator learning methods like DeepONet \cite{lu2021learning}, and Fourier neural operators (FNO) \cite{li2020fourier}. 

Another line of research, starting with \cite{bode2023applying,bode2021using,drygala2022,Fukami:2019,kim2021unsupervised,Kim:2020,King:2017,wang2018esrgan}, suggests the use of generative learning for the fast generation of solutions to turbulent fluid dynamical systems \cite{alkin2024universal,doganay2026learning,klemmer2022spate,oommen2026learning,ross2025world,skorokhodov2022stylegan}. While early work focused on the generation of state snapshots, later works also address the learning of dynamics. This line of work provides extensive numerical evidence for successful learning of chaotic dynamical systems, both in terms of visual similarity and physics-based statistical evaluation. However, there is little to no explanation given as to why these approaches achieve the observed level of performance.

While there exists a substantial body of work dealing with the statistical learning theory of generative models, we cannot directly apply these works to the generative learning of chaotic dynamical systems. In particular, the predominant paradigm of statistical learning theory is learning from independent and identically distributed (i.i.d.) data samples \cite{shalev2014understanding}. But learning from the observation of a single trajectory of a deterministic system, as is often done in physical AI approaches, is certainly far from the i.i.d. assumption. One could resort to learning from multiple trajectories with i.i.d. initial conditions, but this is generally avoided due to considerable burn-in times until CFD solutions develop a statistically stationary turbulent state.

In \cite{drygala2022}, the authors propose replacing the i.i.d. assumption by \emph{ergodicity}, which captures an important statistical property of chaotic dynamical systems. The authors also sketch a mathematical argument to support this approach to learning from physical systems. But no detailed mathematical argument is provided and, in particular, no rates of convergence in the observation time are given.

The goal of this article is to establish a framework in which statistical learning from chaotic but deterministic data is mathematically justified, and to provide rates of convergence in analogy to standard results of statistical learning theory.

Though it is not specific to them, we demonstrate our approach on the example of generative adversarial networks (GANs). GANs were first introduced in \cite{goodfellow2014}. They have since shown remarkable potential, particularly for image generation. Several variants have been developed, including Conditional GAN \cite{mirza2014}, Wasserstein GAN \cite{arjovsky2017}, StyleGAN \cite{karras2021} and more. However, we focus on the original setting of \cite{goodfellow2014}, sometimes referred to as \enquote{vanilla GAN}.

The statistical learning theory of GANs was studied, for instance, in \cite{biau2020}, which links them to the Jensen-Shannon divergence and provides a first proof of consistency. Later papers on the topic include \cite{asatryan2023,liang2021,puchkin2024}. These papers succeed in proving quantitative rates of convergence in various finite- and infinite-dimensional settings, always under the i.i.d. assumption, which enables the use of the law of large numbers, the central limit theorem and other classical results of probability theory.

One possible weaker assumption is the aforementioned ergodicity, which is also used in \cite{drygala2022}. A measure-preserving dynamical system is ergodic if and only if any integrable observable of its states satisfies the strong law of large numbers in the following sense: Let $(M,\cA)$ be a measurable space and $(X_n)_{n\in\N}$ the $M$-valued sequence of states of the system, initialized according to its invariant distribution $\mu = \P \circ X_1^{-1}$. Then for every $f\in L^1(M,\cA,\mu)$, we have
\begin{equation*}
    \frac 1n \sum_{k=1}^n f(X_k) \xrightarrow[n\to\infty]{} \E_{X\sim\mu}[f(X)] \quad \text{almost surely}.
\end{equation*}
This result is known as Birkhoff's pointwise ergodic theorem. While such a generalization of the law of large numbers is sufficient to prove consistency, we require stronger results to prove quantitative rates of convergence. In this, we follow \cite{asatryan2023}, where the authors use McDiarmid's inequality, an exponential concentration inequality for functions of the first $n$ entries of an i.i.d. sequence. Our approach requires dynamical systems that satisfy a similar exponential concentration result, and we use \emph{Young towers} as a representative example. Under this assumption on the underlying dynamical system, we prove the same $\cO(n^{-\tfrac 12})$ rate of almost sure convergence with respect to the Jensen-Shannon divergence as in the i.i.d. setting.

\section{Generative Adversarial Learning}\label{sec:gal}

Given a sequence of training data samples $(Y_i)_{i\in\N}$ following some distribution $\mu\in \cM_1^+(\R^d)$, the goal of generative learning is to be able to generate new samples $(X_i)_{i\in\N}$ following the same or a similar distribution. In other words, to sample from a learned representation of $\mu$.

In generative adversarial learning (GAL), we are trying to learn a transport map $\phi$ (the \emph{generator}) from a noise measure $\nu$ to an approximation $\phi_*\nu=\nu\circ \phi^{-1}$ of $\mu$. The key feature that makes the learning adversarial is that, instead of scoring the generator's output using a difficult to compute negative log-likelihood loss, one introduces a new map $\xi$ (the \emph{discriminator}), which is trained along with the generator in a minimax game, see Figure~\ref{fig:GAN-illustration}. 

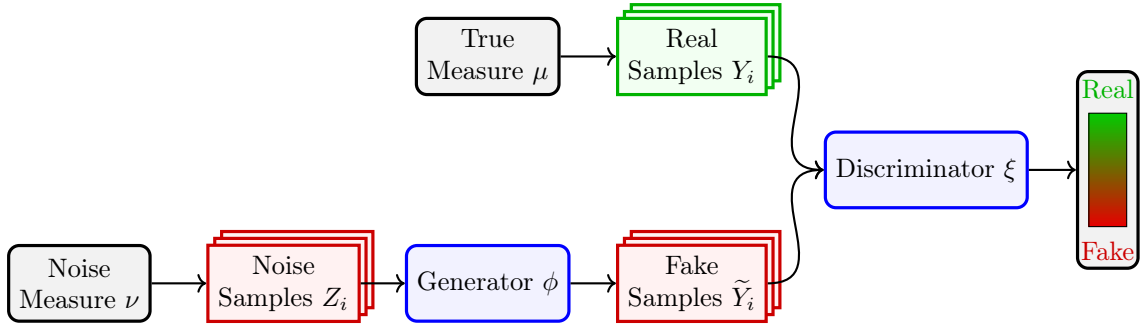
\begin{figure}[ht]
    \centering
    \begin{tikzpicture}[
        every node/.style={align=center},
        font=\footnotesize,
        distribution/.style={rectangle, rounded corners, draw=black, fill=black!5, very thick, minimum height=1cm},
        network/.style={rectangle, rounded corners, draw=blue, fill=blue!5, very thick, minimum height=1cm},
        sample/.style={rectangle, double copy shadow, very thick, minimum height=1cm},
        results/.style={rectangle, rounded corners, draw=black, fill=black!5, very thick, minimum height=26mm,minimum width=8mm},
        ]
    
        \node[distribution] (noise_distribution) at (0,0) {Noise \\ Measure $\nu$};
        \node[sample, draw=black!20!red, fill=red!5] (noise_samples) at (2.7,0) {Noise \\ Samples $Z_i$};
        \node[distribution] (true_distribution) at (5.4,3) {True \\ Measure $\mu$};
        \node[network] (generator) at (5.4,0) {Generator $\phi$};
        \node[sample, draw=black!30!green, fill=green!5] (true_samples) at (8.1,3) {Real \\ Samples $Y_i$};
        \node[sample, draw=black!20!red, fill=red!5] (generated_samples) at (8.1,0) {Fake \\ Samples $\widetilde Y_i$};
        \node[network] (discriminator) at (11.2,1.5) {Discriminator $\xi$};
        
        \node[results] (result_box) at  (13.6,1.5) {};
        \node[rectangle, draw=black, top color=black!20!green, bottom color=black!10!red, minimum height=15mm,minimum width=5mm] at (result_box) {};
        \node at (result_box)[above=8mm] {\color{black!30!green} Real};
        \node at (result_box)[below=8mm] {\color{black!20!red} Fake};
    
        \draw[->, thick] (noise_distribution.east) -- (noise_samples.west);
        \draw[->, thick] (noise_samples.east) -- (generator.west);
        \draw[->, thick] (true_distribution.east) -- (true_samples.west);
        \draw[->, thick] (generator.east) -- (generated_samples.west);
        \draw[->,thick] (true_samples.east) .. controls +(right:10mm) and +(left:10mm) .. (discriminator.west);
        \draw[->,thick] (generated_samples.east) .. controls +(right:10mm) and +(left:10mm) .. (discriminator.west);
        \draw[->,thick] (discriminator.east) -- (result_box.west);
    \end{tikzpicture}
    \caption{Illustration of generative adversarial learning. The generator $\phi$ transforms noise samples $Z_i\sim \nu$ into fake samples $\widetilde Y_i=\phi(Z_i)\sim \phi_*\nu$. The real samples $Y_i$ are generally taken from some training dataset with distribution $\mu$. The discriminator $\xi$ transforms samples into a \enquote{credibility score} between 0 and 1.}
    \label{fig:GAN-illustration}
\end{figure}

The credibility score assigned by the discriminator takes the role of the log-likelihood. The generator aims to maximize the credibility assigned to its fake samples $\phi(Z_i)$ generated from noise. And the discriminator aims to maximize the credibility assigned to true data samples $Y_i$ while minimizing the credibility assigned to fake samples $\phi(Z_i)$. The upside here is that the loss functions can be expressed in terms of $\phi$ and $\xi$ directly, without taking any inverses or derivatives (beyond the usual first derivative for gradient descent).

In the following, we ignore neural networks and gradient descent. Instead, we focus on the theoretical underpinnings of the method, which is why we speak of generative adversarial \emph{learning}, as opposed to \emph{networks}. In fact, we later consider infinite-dimensional hypothesis spaces, which neural networks can only approximate.

In doing so, we closely follow \cite{asatryan2023} for the remainder of this section. In particular, the ERM error decomposition result for generative adversarial learning is theirs, albeit with some modifications to allow for separate treatment of the real data and the noise samples, since the loss of the independence assumption on the former represents the primary contribution of this paper.

The quality of the approximation is measured by the Jensen-Shannon divergence, which is strongly related to the above algorithm, as shown by Proposition~\ref{prop:optimal-discriminators}. To define it, we first need to recall the definition of the Kullback-Leibler divergence.

\begin{definition}[Kullback-Leibler divergence] $ $\\
    Let $\varnothing\neq B\in \cB(\R^d)$ and $\mu,\nu\in \cM_1^+(B)$. The \emph{Kullback-Leibler divergence} between $\mu$ and $\nu$ is
    \begin{equation*}
        \dkl(\mu \| \nu) = \begin{cases}
            \int \frac{d\mu}{d\nu}\log \left(\frac{d\mu}{d\nu}\right) \, d\nu & \text{if} \ \mu \ll \nu,\\
            \infty & \text{else},
        \end{cases}
    \end{equation*}
    where $\frac{d\mu}{d\nu}$ denotes the Radon-Nikodym derivative of $\mu$ w.r.t. $\nu$ and $\mu \ll \nu$ denotes that $\mu$ is absolutely continuous with respect to $\nu$.
\end{definition}

\begin{definition}[Jensen-Shannon divergence] $ $\\
    Let $\varnothing\neq B\in \cB(\R^d)$ and $\mu,\nu\in \cM_1^+(B)$. The \emph{Jensen-Shannon divergence} between $\mu$ and $\nu$ is
    \begin{equation*}
        \djs(\mu,\nu) = \frac{1}{2} \left[\dkl\left(\mu \,\middle\|\, \frac{\mu+\nu}{2}\right) + \dkl\left(\nu \,\middle\|\, \frac{\mu+\nu}{2}\right)\right].
    \end{equation*}
\end{definition}

For simplicity's sake, we restrict ourselves to data on the $d$-dimensional unit cube $[0,1]^d$. Especially in the context of image generation, this is not unreasonable, as any compact range of color values can be rescaled to fit into $[0,1]$. Most arguments transfer directly to compact and convex sets $A\subseteq \R^d$ with nonempty interior. However, in that case, the explicit hypothesis spaces in Section~\ref{sec:infinite-dim-model}, particularly Theorem~\ref{thm:hypothesis-sufficiency}, require significant modification.

\begin{assumption}[general setting for GAL]\label{ass:general-setting} $ $\\
    Let $d\in \N$. We consider the \emph{true measure} $\mu\in \cM_1^+([0,1]^d)$ and the \emph{true data process} $(Y_i)_{i\in\N}$ with $Y_i\sim \mu$ for all $i\in\N$, the \emph{noise measure} $\lambda = \cU([0,1]^d)$ and the i.i.d. \emph{noise process} $(Z_i)_{i\in\N}$ with $Z_i\sim \lambda$ for all $i\in \N$. Generators are chosen from the \emph{hypothesis space of generators} $\cH^G \subseteq \{ \phi : [0,1]^d\to [0,1]^d \ \text{measurable}\}$ and for each $\phi\in \cH^G$, we define the push-forward measure $\mu_\phi = \phi_*\lambda$. Discriminators are chosen from the \emph{hypothesis space of discriminators} $\cH^D \subseteq \{ \xi : [0,1]^d\to (0,1) \ \text{measurable}\}$.
    The \emph{risk function} $\cL: \cH^G \times \cH^D \to \R$ is given by
    \begin{equation*}
        \cL(\phi,\xi) = \frac 12 \Bigl[\E_{Y\sim \mu}\bigl[\log \xi(Y)\bigr] + \E_{Z\sim \lambda}\bigl[\log\bigl(1-\xi(\phi(Z))\bigr)\bigr]\Bigr]
    \end{equation*}
    and the corresponding \emph{minimax problem} by
    \begin{equation}\label{eq:minimax-problem}
        \min_{\phi\in \cH^G} \max_{\xi\in \cH^D} \cL(\phi,\xi).
    \end{equation}
    The \emph{empirical risk function} $\hat\cL : \cH^G \times \cH^D \times \N \times \Omega \to \R$ is given by
    \begin{equation*}
        \hat\cL (\phi, \xi, n) = \frac{1}{2n} \sum_{i=1}^n \log \xi(Y_i) + \frac{1}{2n} \sum_{i=1}^n \log\bigl(1-\xi(\phi(Z_i))\bigr)
    \end{equation*}
    and for each $n\in\N$, the corresponding \emph{empirical minimax problem} is
    \begin{equation}\label{eq:empirical-minimax-problem}
        \min_{\phi\in \cH^G} \max_{\xi\in \cH^D} \hat\cL(\phi,\xi,n).
    \end{equation}
\end{assumption}

Note that for any $\phi\in \cH^G,\, \xi\in \cH^D$ and $n\in\N$, the empirical risk $\hat\cL(\phi,\xi,n)$ is a random variable. Its value depends on the values of the first $n$ samples $(Y_1,\dotsc,Y_n)$ and $(Z_1,\dotsc,Z_n)$.

As announced earlier, we want to establish a connection between the risk function and the Jensen-Shannon divergence. To this end, we require a further assumption on $\mu$, $\cH^G$ and $\cH^D$.

\begin{assumption}[existence of strictly positive densities]\label{ass:densities} $ $\\
    In the setting of Assumption~\ref{ass:general-setting}, we additionally require $\mu \ll \lambda$ with $f_\mu = \frac{d\mu}{d\lambda}$ strictly positive on $[0,1]^d$ and $\mu_\phi \ll \lambda$ with $f_\phi = \frac{d\mu_\phi}{d\lambda}$ strictly positive on $[0,1]^d$, for all $\phi\in \cH^G$.
\end{assumption}

Under this assumption, the best possible discriminator (among all measurable functions, not just $\cH^D$) for any given generator $\phi\in \cH^G$ is known, as is the maximal risk it attains. This result is originally due to \cite[Chapter~2]{biau2020}, though we state it here in the notation of \cite{asatryan2023}.

\begin{proposition}[optimal discriminators]\label{prop:optimal-discriminators} $ $\\
    Let Assumption~\ref{ass:densities} hold and let $\cD = \{ \xi: [0,1]^d \to (0,1) \ \text{measurable}\}$ be the space of all feasible discriminators. Then for $\phi\in\cH^G$, $\xi_\phi = \frac{f_\mu}{f_\mu+f_\phi} \in \cD$ is well-defined and is the $\lambda$-a.s. unique solution to $\max_{\xi \in \cD} \cL(\phi,\xi) = \djs(\mu,\mu_\phi) - \log(2)$.
\end{proposition}

If we had $\{\xi_\phi \mid \phi \in \cH^G\} \subseteq \cH^D$, the search for a generator $\phi$ that minimizes $\djs(\mu,\mu_\phi)$ would relate to the minimax problem \eqref{eq:minimax-problem} via $\argmin_{\phi\in \cH^G} \djs(\mu,\mu_\phi) = \argmin_{\phi\in\cH^G} \max_{\xi\in\cH^D} \cL(\phi,\xi)$.

The empirical minimax problem \eqref{eq:empirical-minimax-problem} is an approximation of \eqref{eq:minimax-problem}, so it stands to reason that solving the former also yields an approximate solution to the latter. The primary goal of both \cite{asatryan2023} and this thesis is the quantification of this approximation error and proving that -- under additional assumptions -- it converges to $0$ in the large sample limit $n\to\infty$. To this end, we first separate the risk function into two components.

\begin{definition}[separate risk functions] $ $\\
    We define the separate (empirical) risk functions 
    \begin{align*}
        \cL^\mu(\xi) &= \E_{Y\sim \mu}\bigl[\log \xi(Y)\bigr] & \cL^\lambda(\phi,\xi) &= \E_{Z\sim \lambda}\bigl[\log\bigl(1-\xi(\phi(Z))\bigr)\bigr]\\
        \hat \cL^\mu(\xi,n) &= \frac{1}{n}\sum_{i=1}^n \log \xi(Y_i) & \hat\cL^\lambda(\phi,\xi,n) &= \frac{1}{n} \sum_{i=1}^n \log\bigl(1-\xi(\phi(Z_i))\bigr)
    \end{align*}
    for data sampled from the true measure $\mu$, respectively the noise measure $\lambda$.
\end{definition}

Finally, we can decompose the aforementioned approximation error into several terms that are easier to control separately.

\begin{proposition}[error decomposition]\label{prop:error-decomposition} $ $\\
    Let Assumption~\ref{ass:densities} hold. For $n\in\N$ fixed, let $\hat\phi_n$ solve \eqref{eq:empirical-minimax-problem}. Then
    \begin{equation*}
        \djs(\mu,\mu_{\hat\phi_n}) \leq \emodelg + \emodeld + \egenmu(n) + \egenlambda(n),
    \end{equation*}
    where the \emph{model errors} are defined as
    \begin{align*}
        \emodelg &= \inf_{\phi\in \cH^G} \djs(\mu,\mu_\phi), &
        \emodeld &= \sup_{\phi\in \cH^G} \inf_{\xi\in \cH^D} \cL(\phi,\xi_\phi) - \cL(\phi,\xi),
    \intertext{and the \emph{generalization errors} are defined as}
        \egenmu(n) &= \sup_{\xi\in \cH^D} \Abs{ \hat\cL^\mu(\xi,n) - \cL^\mu(\xi) }, &
        \egenlambda(n) &= \sup_{\substack{\phi\in \cH^G \\ \xi\in \cH^D}} \Abs{ \hat\cL^\lambda(\phi,\xi,n) - \cL^\lambda(\phi,\xi) }.
    \end{align*}
\end{proposition}

\begin{proof}
    This proof is adapted from \cite[Theorem~2.3]{asatryan2023}. Let
    \begin{equation*}
        \egen(n) = \sup_{\substack{\phi\in \cH^G \\ \xi\in \cH^D}} \Abs{\hat\cL(\phi,\xi,n) - \cL(\phi,\xi)}.
    \end{equation*}
    
    Let $\phi^*\in \cH^G$ and $\xi^*\in \cH^D$ be arbitrary and fixed. Then we have
    \begin{equation*}
        \cL(\hat  \phi_n, \xi_{\hat\phi_n}) = \cL(\hat\phi_n,\xi^*) + \cL(\hat  \phi_n, \xi_{\hat\phi_n}) - \cL(\hat\phi_n,\xi^*) \leq \sup_{\xi\in\cH^D} \cL(\hat\phi_n,\xi) + \cL(\hat  \phi_n, \xi_{\hat\phi_n}) - \cL(\hat\phi_n,\xi^*).
    \end{equation*}
    Since $\xi^*\in \cH^D$ was arbitrary, taking the infimum over $\xi^*\in \cH^D$ on the right yields
    \begin{align*}
        \cL(\hat  \phi_n, \xi_{\hat\phi_n}) &\leq \sup_{\xi\in\cH^D} \cL(\hat\phi_n,\xi) + \emodeld\\
        &\leq \sup_{\xi\in \cH^D} \hat\cL(\hat\phi_n,\xi,n) + \emodeld + \egen(n)\\
        &= \min_{\phi\in \cH^D} \sup_{\xi\in \cH^D} \hat\cL(\phi,\xi,n) + \emodeld + \egen(n)\\
        &\leq \sup_{\xi\in \cH^D} \hat\cL(\phi^*,\xi,n) + \emodeld + \egen(n)\\
        &\leq \sup_{\xi\in \cH^D} \cL(\phi^*,\xi) + \emodeld + 2\egen(n) \\
        &\leq \cL(\phi^*,\xi_{\phi^*}) + \emodeld + 2\egen(n).
    \end{align*}
    By Proposition~\ref{prop:optimal-discriminators}, adding $\log(2)$ on both sides yields
    \begin{equation*}
        \djs(\mu,\mu_{\hat\phi_n}) \leq \djs(\mu,\mu_{\phi^*}) + \emodeld + 2\egen(n).
    \end{equation*}
    Since $\phi^*\in \cH^G$ was arbitrary, taking the infimum over $\phi^*\in \cH^G$ on the right yields
    \begin{equation*}
        \djs(\mu,\mu_{\hat\phi_n}) \leq \emodelg + \emodeld + 2\egen(n).
    \end{equation*}
    Finally, the triangle inequality immediately yields $2\egen(n) \leq \egenmu(n) + \egenlambda(n)$.
\end{proof}

While the model errors are deterministic, the generalization errors are random variables, since they depend on the empirical risk functions $\hat \cL^\mu$ respectively $\hat\cL^\lambda$, which in turn depend on the first $n$ samples $(Y_1,\dotsc,Y_n)$ respectively $(Z_1,\dotsc,Z_n)$. Note also that the generalization errors are the suprema of continuous stochastic processes over separable sets and thus measurable.

Towards proving that $\djs(\mu,\mu_{\hat\phi_n}) \to 0$ as $n\to \infty$, we must first eliminate the model errors. To this end, in Section~\ref{sec:infinite-dim-model}, we provide stronger assumptions on the true measure $\mu$ and an explicit choice of hypothesis spaces $\cH^G$ and $\cH^D$ that satisfy Assumption~\ref{ass:densities} and yield no model error.

Afterwards, in Section~\ref{sec:rates}, we prove that the two generalization errors also converge to zero almost surely with an $\cO(n^{-\frac 12})$ (up to a polylogarithm) rate, under an additional assumption that the training data is sampled from a suitable dynamical system, which we detail in Section~\ref{sec:young-towers}.

\section{An Infinite-Dimensional Model of Generative Adversarial Learning}\label{sec:infinite-dim-model}

We consider once again the setting of generative adversarial learning outlined in Assumption~\ref{ass:general-setting}. Throughout this section, the dimension $d\in \N$ as well as the regularity $k\in \N$, $k\geq 2$ and the Hölder coefficient $\alpha\in (0,1]$ are considered fixed.

Our goal is to give assumptions on $\mu$ and a corresponding explicit choice of $\cH^G$ and $\cH^D$ such that the empirical risk minimizer $\hat\phi_n$ always exists, the error decomposition of Proposition~\ref{prop:error-decomposition} is applicable, and we have $\emodelg=\emodeld=0$.

We begin with a regularity assumption on the density $f_\mu$ of $\mu$. However, unlike Assumption~\ref{ass:densities}, we do not yet make any assumptions about the existence of generated densities $f_\phi$.

\begin{assumption}[regularity of the true measure]\label{ass:true-measure-hoelder} $ $\\
    In the setting of Assumption~\ref{ass:general-setting}, we additionally require $\mu \ll \lambda$ with $f_\mu = \frac{d\mu}{d\lambda} \in C^{k,\alpha}([0,1]^d,\R)$, where $k\geq 2$ and $\alpha\in (0,1]$, and that $f_\mu$ be bounded away from zero by $\kappa = \min_{y\in [0,1]^d} f_\mu(y) > 0$.
\end{assumption}

In the following, we present a class of Hölder differentiable hypothesis spaces $\cH^G$ and $\cH^D$ from \cite{asatryan2023}, which meet the criteria outlined above if the true measure fulfills Assumption~\ref{ass:true-measure-hoelder}, starting with the generators. The definition of the spaces $C^{k,\alpha}$ is recapped in Appendix~\ref{ap:hoelder}.

\begin{definition}[hypothesis space of generators] $ $\\
    Let $K,\hat K > 0$. We define $\chg$ as the set of all $\phi: [0,1]^d \to [0,1]^d$ where
    \begin{itemize}
        \item $\phi\in C^{k,\alpha}([0,1]^d,\R^d)$ with $\norm{\phi}_{C^{k,\alpha}} \leq K$.
        \item $\phi$ is bijective and $\phi^{-1}\in C^{k,\alpha}([0,1]^d,\R^d)$ with $\norm{\phi^{-1}}_{C^{k,\alpha}} \leq \hat K$.
    \end{itemize}
\end{definition}

The hypothesis space of discriminators is chosen to match the hypothesis space of generators, in order to obtain optimal discriminators for all generators.

\begin{definition}[hypothesis space of discriminators] $ $\\
    Let $B\in (0,\tfrac 12)$ and $C_1,C_2>0$. We define $\chd$ as the set of all $\xi\in [0,1]^d \to (0,1)$ where
    \begin{itemize}
        \item $\xi$ is bounded away from zero and one by $B\leq \xi \leq 1-B$.
        \item The Fréchet derivative of $\xi$ satisfies $\norm{D\xi}_\infty \leq C_1$.
        \item $\xi\in C^{k-1,\alpha}([0,1]^d,\R)$ with $\norm{\xi}_{C^{k-1,\alpha}}\leq C_2$.
    \end{itemize}
\end{definition}

Assumption~\ref{ass:true-measure-hoelder} already implies that $\mu$ has a strictly positive Lebesgue density. The next results, which are proven in \cite[Chapter~3]{asatryan2023}, show that each generated measure $\mu_\phi$ does as well, and that the model errors vanish in Proposition~\ref{prop:error-decomposition}.

\begin{proposition}[existence of strictly positive densities for $\cH^G$]\label{prop:generator-densities} $ $\\
    Let $K,\hat K > 0$. For each $\phi\in \chg$, we have $\mu_\phi \ll \lambda$ and its Lebesgue density is given by $f_\phi(y) = \abs{D\phi^{-1}(y)} \in C^{k-1,\alpha}([0,1]^d,\R)$. Consequently, Assumption~\ref{ass:densities} is fulfilled.
\end{proposition}

\begin{theorem}[sufficiency of $\cH^G$ and $\cH^D$]\label{thm:hypothesis-sufficiency} $ $\\
    Let Assumption~\ref{ass:true-measure-hoelder} hold. There exists a function $\phi_\mu \in C^{k,\alpha}([0,1]^d,\R^d)$ with $(\phi_\mu)_*\lambda = \mu$ and if $K,\hat K$ are sufficiently large, we have $\phi_\mu\in \chg$.

    If, additionally, $C_1, C_2$ are sufficiently large and $B$ is sufficiently small, then the optimal discriminators satisfy $\xi_{\phi,\phi'} = \frac{f_{\phi}}{f_{\phi}+f_{\phi'}} \in \chd$ for all $\phi,\phi'\in\chg$. In particular, we have $\xi_\phi = \frac{f_\mu}{f_\mu + f_\phi} \in \chd$ for all $\phi\in \chg$. 
    
    Altogether, given appropriate choice of $K,\hat K, B, C_1,C_2$, we have $\emodelg = \emodeld = 0$.
\end{theorem}

We formalize this model as another assumption.

\begin{assumption}[Hölder model]\label{ass:model-hoelder}$ $\\
    Let Assumption~\ref{ass:true-measure-hoelder} hold. We additionally require that $\cH^G = \chg$ and $\cH^D = \chd$ with the constants $K,\hat K, B,C_1,C_2$ chosen as in Theorem~\ref{thm:hypothesis-sufficiency}.
\end{assumption}

It remains to prove that the empirical risk minimizer $\hat\phi_n$ actually exists. For this, we require the compactness of the hypothesis spaces and continuity of the target functions. Since the Hölder spaces are infinite-dimensional, we cannot expect compactness in their native topologies. However, the embedding into a slightly weaker $C^{k,\alpha'}$-topology, for any fixed $\alpha' \in (0,\alpha)$, is proven to be compact in \cite[Chapter~3]{asatryan2023}.

The following result proves not just continuity but even Lipschitz continuity of the (empirical) risk functions, and will be important in Section~\ref{sec:rates} as well. 

\begin{proposition}[Lipschitz continuity of the target functions]\label{prop:target-lipschitz-continuity}$ $\\
    Let Assumption~\ref{ass:model-hoelder} hold. For $n\in\N$, $\phi,\phi' \in \chg$ and $\xi,\xi'\in \chd$ we have
    \begin{align*}
        \Abs{\cL^\mu(\xi) - \cL^\mu(\xi')} &\leq \frac{1}{B} \norm{\xi-\xi'}_\infty, \\
        \Abs{\hat \cL^\mu(\xi,n)-\hat \cL^\mu(\xi',n)} &\leq \frac{1}{B} \norm{\xi-\xi'}_\infty, \\
        \Abs{\cL^\lambda(\phi,\xi) - \cL^\lambda(\phi',\xi')} &\leq \frac{1}{B} \bigl(\norm{\xi-\xi'}_\infty + C_1 \norm{\phi-\phi'}_\infty\bigr), \\
        \Abs{\hat \cL^\lambda(\phi,\xi,n)-\hat \cL^\lambda(\phi',\xi',n)} &\leq \frac{1}{B} \bigl(\norm{\xi-\xi'}_\infty + C_1 \norm{\phi-\phi'}_\infty\bigr).
    \end{align*}
    The same holds when $\norm{\cdot}_\infty$ is replaced by any $C^{\widetilde k}$ or $C^{\widetilde k,\widetilde \alpha}$-norm, $\widetilde k\in \N_0$, $\widetilde\alpha \in (0,1]$.
\end{proposition}

\begin{proof}
    The proof is analogous to \cite[Lemma~4.1]{asatryan2023}.
\end{proof}

Using this result, we can prove even more than just the existence of empirical risk minimizers.

\begin{theorem}[existence of solutions for \eqref{eq:minimax-problem} and \eqref{eq:empirical-minimax-problem}]\label{thm:erm-existence} $ $\\
    Let Assumption~\ref{ass:model-hoelder} hold. Then for all $n\in \N$, the minima and maxima in the minimax problem \eqref{eq:minimax-problem} and the empirical minimax problem \eqref{eq:empirical-minimax-problem} are attained. In particular, the empirical risk minimizer $\hat\phi_n$ from Proposition~\ref{prop:error-decomposition} always exists.
\end{theorem}

\begin{proof}
    Let $0<\alpha'<\alpha$, then by Proposition~\ref{prop:target-lipschitz-continuity} we have
    \begin{align*}
        \Abs{\cL(\phi,\xi) - \cL(\phi',\xi')} &\leq \frac{1}{B} \norm{\xi-\xi'}_{C^{k-1,\alpha'}} + \frac{C_1}{2B} \norm{\phi-\phi'}_{C^{k,\alpha'}}, \\
        \Abs{\hat \cL(\phi,\xi,n)-\hat \cL(\phi',\xi',n)} &\leq \frac{1}{B} \norm{\xi-\xi'}_{C^{k-1,\alpha'}} + \frac{C_1}{2B} \norm{\phi-\phi'}_{C^{k,\alpha'}}.
    \end{align*}
    By \cite[Proposition~3.6]{asatryan2023} and \cite[Proposition~3.11]{asatryan2023}, $\chg$ and $\chd$ are compact in the $C^{k,\alpha'}$- and $C^{k-1,\alpha'}$-topologies respectively. The assertion follows because all minima and maxima are attained for a minimax problem with Lipschitz continuous target function.
\end{proof}

Since Assumption~\ref{ass:densities} is fulfilled by Proposition~\ref{prop:generator-densities}, and the existence of an empirical risk minimizer is guaranteed by Theorem~\ref{thm:erm-existence}, we have $\djs(\mu,\mu_{\hat\phi_n}) \leq \egenmu(n) + \egenlambda(n)$ by Proposition~\ref{prop:error-decomposition} and Theorem~\ref{thm:hypothesis-sufficiency}. It remains to prove that these generalization errors converge to zero with suitable rates. In order to do so for $\egenmu(n)$, we finally need to specify our assumptions on the true data process $(Y_i)_{i\in\N}$.

\section{A Chaotic Deterministic Data Model}\label{sec:young-towers}

\subsection{Ergodic Dynamical Systems}

The motivation for this paper is learning the invariant distribution of a chaotic dynamical system from a single time series, as in \cite{drygala2022}. This requires said dynamical system to satisfy conditions that can substitute for the law of large numbers and related results of the classical i.i.d. setting. A simple class of such systems turns out to be that of ergodic dynamics. The definitions below follow \cite[Chapters~1 \& 5]{eisner2015}.

\begin{definition}[dynamical system; measure-preserving system] $ $\\
    A \emph{dynamical system} $(M,\cA,T)$ consists of a measurable space $(M,\cA)$ together with a measurable map $T:M\to M$ called the \emph{dynamics}.
    
    A \emph{measure-preserving system} $(M,\cA,\mu,T)$ is a dynamical system $(M,\cA,T)$ endowed with a probability measure $\mu\in \cM_1^+(M,\cA)$ such that $\mu = \mu\circ T^{-1}$. We call $\mu$ an \emph{invariant measure} for $(M,\cA,T)$.
\end{definition}

\begin{definition}[invariant $\sigma$-algebra; ergodicity] $ $\\
    Let $(M,\cA,\mu,T)$ be a measure-preserving system. The \emph{invariant $\sigma$-algebra} is
    \begin{equation*}
        \cI_T = \{ A\in \cA \mid A = T^{-1}(A)\}.
    \end{equation*}
    We call the system \emph{ergodic} if $\cI_T$ is $\mu$-trivial, i.e., $\mu(A)\in \{0,1\}$ for all $A\in\cI_T$.

    In this case we also call $T$ an \emph{ergodic transformation} on $(M,\cA,\mu)$.
\end{definition}

We now fix a probability space $(\Omega,\cF,\P)$. Let any random variable or stochastic process whose domain is not explicitly specified be defined on this space.

If $\X=(X_n)\nin$ is a stochastic process taking values in a measurable space $(M,\cA)$, we refer to $\P\circ \X^{-1}\in \cM_1^+(M^\N,\cA^{\otimes\N})$ as the (joint) distribution of the process $\X$.

\begin{definition}[ergodic process] $ $\\
    Let $(M,\cA)$ be a measurable space and $\X=(X_n)\nin$ an $M$-valued stochastic process. $\X$ is called an \emph{ergodic process} if the left-shift
    \begin{equation*}
        \tau:  M^\N \to  M^\N,\quad (x_n)\nin \mapsto (x_{n+1})\nin
    \end{equation*}
    is an ergodic transformation on $( M^\N, \cA^{\otimes\N},\P\circ \X^{-1})$.
\end{definition}

\begin{remark}[ergodic processes of observables]\label{rem:observable-processes} $ $\\
    Ergodic processes can be defined in a very natural manner on ergodic systems. Let $(M,\cA,\mu,T)$ be an ergodic system, $(\widetilde M,\widetilde\cA)$ a measurable space and $g: M\to \widetilde M$ a measurable observable of the system. We define the $\widetilde M$-valued process $\X=(X_n)\nin$ of values of the observable $g$ under repeated application of $T$ by $X_{n} = g\circ T^{n-1}$. Then $\X$ is an ergodic process.

    As a special case of the above, we obtain what we refer to as the \emph{canonical process} of the system $(M,\cA,\mu,T)$. Let $g=\Id_M$, then we have $X_1(x)=x$ and $X_{n+1}(x)=T(X_n)$, i.e., the (deterministic) process of states of the system under repeated application of $T$.
\end{remark}

In \cite{asatryan2023}, the authors use the uniform law of large numbers to prove that $\egen(n)\to 0$ almost surely as $n\to\infty$. At this point, if we simply required $(Y_i)_{i\in\N}$ to be an ergodic process, we could already prove consistency along a similar line, using Birkhoff's ergodic theorem together with the Lipschitz continuity of Proposition~\ref{prop:target-lipschitz-continuity}. See \cite{drygala2022} for a sketch of such an argument.

To prove not just consistency but quantitative rates of convergence, however, we require stronger results than the ergodic theorem. The proof in \cite{asatryan2023} utilizes McDiarmid's inequality (and its special case, Hoeffding's inequality), but there is no analogous statement for general ergodic systems. While there are several narrower classes of dynamical systems that satisfy such exponential concentration inequalities -- some others, we mention in Section~\ref{sec:conclusion} -- we focus here on one such class to exemplify our approach.

\subsection{Young Towers}

For this section, let $(M,g)$ be a smooth Riemannian manifold, $\d$ the distance on it, $\cA=\cB(M)$ the Borel-$\sigma$-algebra and $\mu$ the Riemannian volume on $M$. Such Riemannian manifolds include $\R^n$, the $(n-1)$-dimensional unit sphere $\S^{n-1}$ in $\R^n$, and the $n$-torus $\T^n=\R^n / \Z^n$ \cite{lee2018}. We consider dynamical systems $(M,\cA,T)$, where $T$ is either a $C^k$-diffeomorphism with $k\geq 2$ or a $C^{1,\alpha}$-diffeomorphism for some $\alpha\in (0,1]$. We collectively refer to either as a \emph{$C^{1+\epsilon}$-diffeomorphism}. Since the $\sigma$-algebra is always the Borel-$\sigma$-algebra, we usually omit $\cA$ and simply speak of dynamical systems $(M,T)$ or measure-preserving systems $(M,\mu,T)$.

\emph{Young towers} can be seen as a broad generalization of Axiom A systems and their symbolic dynamics (see e.g. \cite[Chapter~3]{bowen2008} for an introduction to Axiom A and Anosov systems) and were first introduced by Lai-Sang Young in \cite{young1998}.

To be precise, a \emph{uniform Young tower} is a dynamical system $(\Delta,\hat T)$ that extends the underlying system $(M,T)$ and which can be more easily studied using methods of symbolic dynamics. Properties like mixing or ergodicity are proven for $(\Delta,\hat T)$, and from there transfer to $(M,T)$.

The first question, then, must be under what conditions it is possible to construct a uniform Young tower over the system $(M,T)$. The exact conditions are listed in Appendix~\ref{ap:young} and only briefly summarized in this section.

\subsubsection*{The Underlying System: Hyperbolic Product Structure and Return Times}

We assume there exists $\Lambda\subseteq M$ with a so-called \emph{hyperbolic product structure}. This means that, though $\Lambda$ is not necessarily a hyperbolic set and not even $T$-invariant, $T$ exhibits hyperbolic behavior on $\Lambda$ with clearly defined stable and unstable directions in every point. For a more detailed account, see Definition~\ref{def:hyperbolic-product-structure}.

Let $(\Lambda_i)_{i\in I}$ be a partition of $\Lambda$, where either $I=\{1,\dotsc,n\}$ for some $n\in\N$ or $I=\N$. Moreover, for each $i\in I$ let $R_i\in \N$ be the \emph{return time} of $\Lambda_i$, that is, the first time $n\in\N$ where $T^n(\Lambda_i) \subseteq \Lambda$. More precisely, $T^{R_i}(\Lambda_i)$ must be a \emph{$u$-subset} of $\Lambda$, i.e., it spans $\Lambda$ fully in the unstable direction, as defined by the hyperbolic product structure. See Definition~\ref{def:s-u-subsets} for more details.

Crucially, $R_i$ is not necessarily the first return time for each individual $x\in \Lambda_i$, but rather the first \emph{simultaneous} return time for the entirety of $\Lambda_i$. 

For any individual $x\in \Lambda_i$, its orbit up to time $R_i-1$ may or may not be in $\Lambda$. By time $R_i$, it will land in $\Lambda$, though possibly in a different partition element, i.e., $x'=T^{R_i}(x)\in \Lambda_j$. From there, the cycle begins anew only with $R_j$ instead of $R_i$. 

\subsubsection*{The Abstract System: Uniform Young Tower}

The Young tower constructed over this system is essentially a set of \enquote{virtual points} that represent the orbit of each $\Lambda_i$ outside of $\Lambda$ until its return.

Let $R:\Lambda \to \N$ be the return time function defined by $R|_{\Lambda_i} =R_i$ for all $i\in I$. Let $T^R: \Lambda \to \Lambda$ be the return map with $T^R|_{\Lambda_i}=T^{R_i}|_{\Lambda_i}$ for all $i\in I$. We define the new state space
\begin{equation*}
    \Delta = \bigl\{(x,l) \mid x\in \Lambda,\, l\in \{0,\dotsc,R(x)-1\} \bigr\}
\end{equation*}
and the new dynamics
\begin{equation*}
    \hat T(x,l) = \begin{cases}
        (x,l+1) & \text{if}\ l+1 < R(x),\\
        (T^R(x),0) & \text{if}\ l+1=R(x).
    \end{cases}
\end{equation*}
We think of $\Delta$ as a tower, with its $l$-th \enquote{level} defined as
\begin{equation*}
    \Delta_l = \{(x,l) \mid x\in \Lambda,\, R(x)>l\}.
\end{equation*}
We further define the embeddings
\begin{equation*}
    \iota_l : \{x\in \Lambda \mid R(x)>l\} \to \Delta_l, \quad x\mapsto (x,l).
\end{equation*}
See Figure~\ref{fig:symbolic-young-tower} for an illustration of such a tower. Note that $\Delta_l$ consists of copies $\iota_l(\Lambda_i)$ of each $\Lambda_i$ with $R_i>l$. The dynamics $\hat T$ map $\iota_l(\Lambda_i)$ to $\iota_{l+1}(\Lambda_i)$ if $l+1<R_i$. Otherwise, $\iota_{R_i-1}(\Lambda_i)$ is mapped onto a $u$-subset of $\Delta_0=\iota_0(\Lambda)$, which by definition spans all $\Lambda_j$.

\begin{figure}[ht]
    \centering
    \begin{tikzpicture}[
        font=\footnotesize,
        box/.style={rectangle, draw=black, thick, fill=black!5, minimum height=1cm, minimum width=2cm}
    ]
        \newcommand*\Returns{{2,4,3,4}}
        \newcommand*\towers{4};
        \draw [thick, -Stealth] (0,0) -- (2*\towers+2,0) node [right] {$x$};
        \draw [thick, -Stealth] (0,0) -- (0,5) node [above] {$l$};

        \foreach \l in {0,...,3}
        {
            \draw [decorate,decoration={calligraphic brace,amplitude=0.15cm},line width=0.75pt] (0,\l) -- (0,\l+1) node [midway,xshift=-0.5cm] {$\Delta_{\pgfmathprint{\l}}$};
        }
        
        \foreach[parse=true] \i in {1,...,\towers}
        {
            \draw [decorate,decoration={calligraphic brace,amplitude=0.3cm,mirror},line width=0.75pt] (2*\i-2,0) -- (2*\i,0) node [midway,xshift=0.08cm,yshift=-0.5cm] {$\Lambda_{\pgfmathprint{\i}}$};
            
            \draw (2*\i-1, \Returns[\i-1]+0.5) node {$R_{\pgfmathprint{\i}}=\pgfmathprint{\Returns[\i-1]}$};
            
            \foreach[parse=true] \l in {0,...,0+\Returns[\i-1]-1}
            {   
                \node[box] at (2*\i-1, \l + 0.5) {$\iota_{\pgfmathprint{\l}}(\Lambda_{\pgfmathprint{\i}})$};
            }
        }
        
        \foreach[parse=true] \l in {1,...,0+\Returns[\towers-1]-1}
        {
            \begin{scope}[shift={(2*\towers,\l)}]
                \draw[thick,->] (-90:0.35) arc (-90:90:0.35);
            \end{scope}
        }
        
        \begin{scope}[shift={(2*\towers,\Returns[\towers-1])}]
            \draw[thick,->] (-90:0.2) arc (90:0:1)  node[below, xshift=0.08cm] {$\Delta_0$};
        \end{scope}
        
        \node at (2*\towers+1,0.5) {\large $\cdots$};
    \end{tikzpicture}
    \caption{A uniform Young tower and its dynamics}
    \label{fig:symbolic-young-tower}
\end{figure}
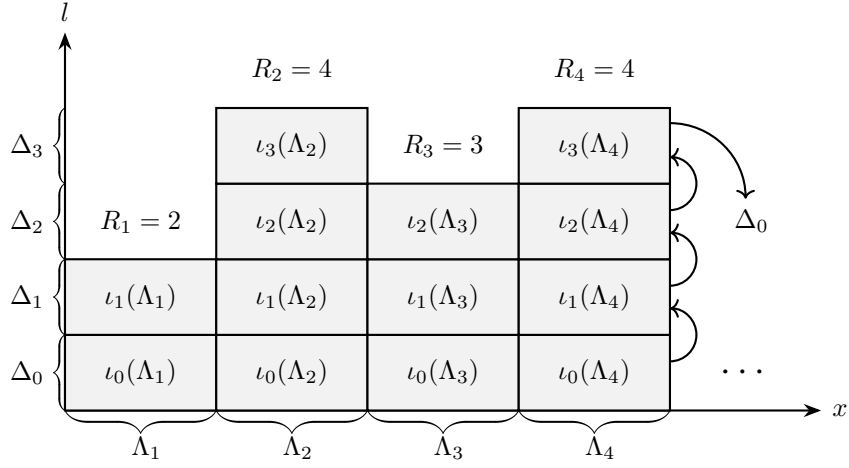

Recall that $\iota_l(\Lambda_i)$ represents the set $T^l(\Lambda_i)$, which may or may not be disjoint from $\Lambda$. We therefore define the projection
\begin{equation*}
    \pi : \Delta \to M,\quad (x,l) \mapsto T^l(x).
\end{equation*}
By definition of $\pi$ and $\hat T$ we have the semi-conjugacy
\begin{equation*}
    T\circ \pi = \pi \circ \hat T.
\end{equation*}
The reason we call points in $\Delta$ virtual is that two points that differ in $\Delta$ may coincide in $M$. Consider, for example, $x\in \Lambda$ with $T(x)\in \Lambda$ as well but $R(x)>1$ (recall that $R(x)=R_i$ is only the first \emph{simultaneous} return time for $\Lambda_i\ni x$). Then $(x,1)$ and $(T(x),0)$ are distinct elements of $\Delta$, but they each represent the same element $T(x)$ of $M$, since $\pi(x,1)=T(x)=\pi(T(x),0)$.

\begin{definition}[uniform Young tower; exponential tails, aperiodic]\label{def:young-tower} $ $\\
    We call the system $(\Delta,\hat T)$ defined above a \emph{uniform Young tower} if the requirements \textbf{(P1)-(P5)} and \textbf{(S)} from Appendix~\ref{ap:young} are fulfilled. In that case, we say that $(M,T)$ \emph{admits} or \emph{is modeled by a uniform Young tower}.

    We say the tower \emph{has exponential tails} if there exist $c> 0$ and $\tau \in (0,1)$ such that
    \begin{equation*}
        \mu(\{ x\in \Lambda \mid R(x)>n\}) \leq c \tau^n
    \end{equation*}
    holds for all $n\in\N$.

    We call the tower \emph{aperiodic} if $\gcd \{ R(x) \mid x\in \Lambda\} = 1$.    
\end{definition}

As promised earlier, certain properties of $\Delta$ transfer to $M$. 

\begin{proposition}[invariant measure]\label{prop:young-tower-invariant-measure} $ $\\
    Let $(M,T)$ be modeled by a uniform Young tower $(\Delta,\hat T)$ with exponential tails. Then there exists a $T$-invariant probability measure $\nu$ on $M$.

    If $(\Delta,\hat T)$ is additionally aperiodic, then $(M,\nu,T)$ is mixing and in particular, ergodic.
\end{proposition}

\begin{proof}
    By \cite[Theorem~1]{young1998}, if $\int_{\Lambda} R \, d\mu < \infty$, then $\Delta$ admits a $\hat T$-invariant probability measure $\hat \nu$. If $R$ has exponential tails, it is also $\mu$-integrable. Let $\nu = \pi_*\hat\nu$. Then
    \begin{equation*}
        \nu \circ T^{-1} = \hat \nu \circ \pi^{-1} \circ T^{-1} = \hat\nu \circ \hat T^{-1} \circ \pi^{-1} = \hat\nu \circ \pi^{-1} = \nu,
    \end{equation*}
    or in other words, $\nu$ is $T$-invariant.

    If the tower is aperiodic, then $\hat \nu$ is mixing by \cite[Theorem~1]{young1999} and thus
    \begin{align*}
        \nu\bigl(A\cap T^{-n}(B)\bigr) &= \hat \nu \bigl(\pi^{-1}(A)\cap \pi^{-1}(T^{-n}(B))\bigr) \\
        &= \hat \nu \bigl(\pi^{-1}(A)\cap \hat T^{-n}(\pi^{-1}(B))\bigr)\\
        &\xrightarrow[n\to\infty]{} \hat\nu \bigl(\pi^{-1}(A)\bigr) \, \hat\nu\bigl(\pi^{-1}(B)\bigr)\\
        &= \nu(A)\, \nu(B)
    \end{align*}
    holds for all $A,B\in \cB(M)$, i.e., $\nu$ is mixing.
\end{proof}

Other properties of aperiodic uniform Young towers with exponential tails and the systems modeled by them include the exponential decay of correlations for observable processes defined over $(M,\nu,T)$ as in Remark~\ref{rem:observable-processes} with Hölder continuous $g$ \cite[Theorem~2]{young1998}, as well as the central limit theorem for the same \cite[Theorem~3]{young1998}.

The most interesting property to us is that the states of a system modeled by an aperiodic uniform Young tower with exponential tails exhibit an exponential concentration inequality similar to McDiarmid's inequality, namely the Chazottes-Gouëzel inequality, Theorem~\ref{thm:mcdiarmid-young}.

It remains to discuss examples of systems that satisfy these conditions. As mentioned earlier, Young towers generalize Axiom A systems. It is thus unsurprising that Axiom A systems make up the first class of examples, per \cite[Theorem~4]{young1998}.

\begin{proposition}[Young towers for Axiom A]\label{prop:axiom-a-young} $ $\\
    If $T$ is at least a $C^2$-diffeomorphism and $(M,T)$ satisfies Axiom A, then it admits a uniform Young tower with exponential tails.
\end{proposition}

As examples of such Axiom A systems, consider the following class.

\begin{example}[hyperbolic torus automorphisms]\label{ex:hta} $ $\\
    Let $d\in \N$, $d\geq 2$. We consider $M=\T^d=\R^d/\Z^d$, the $d$-torus with the usual metric
    \begin{equation*}
        \d_{\T^d}([x],[y]) = \inf_{p\in\Z^d} \abs{x-y+p},
    \end{equation*}
    where $[x]$ denotes the equivalence class of $x\in\R^d$ modulo $\Z^d$. One can think of $\T^d$ as $[0,1]^d$ with the opposing edges identified. Now consider a symmetric matrix $A\in \Z^{d\times d}$ with $\abs{\det(A)}= 1$ and $\abs{\lambda}\neq 1$ for all $\lambda\in \sigma(A)$. We define the associated hyperbolic torus automorphism by
    \begin{equation*}
        T : \T^d \to \T^d,\qquad [x] \mapsto [Ax].
    \end{equation*}
    Here, the fact that all entries are integer guarantees that $T$ is well-defined and by Cramer's rule, $\abs{\det(A)}$ guarantees $A^{-1}\in\Z^{d\times d}$ as well, such that $T^{-1}$ is represented by $A^{-1}$. The additional properties guarantee that $(M,T)$ is an Anosov system and thus Axiom A. See \cite[Proposition~III.21]{zehnder2010} for the 2-dimensional case. The proof in higher dimensions is mostly analogous but makes use of the normality guaranteed by symmetry. Moreover, these hyperbolic torus automorphisms have the Lebesgue measure as their invariant distribution $\nu$ and are mixing with respect to it.

    The most famous example of this class is Arnold's cat map, first introduced in \cite{arnold1968}, whose action on the $2$-torus is pictured in Figure~\ref{fig:cat-map}. Its associated matrix is
    \begin{equation*}
        A = \begin{bmatrix}
            2 & 1 \\ 1 & 1
        \end{bmatrix}.
    \end{equation*}
    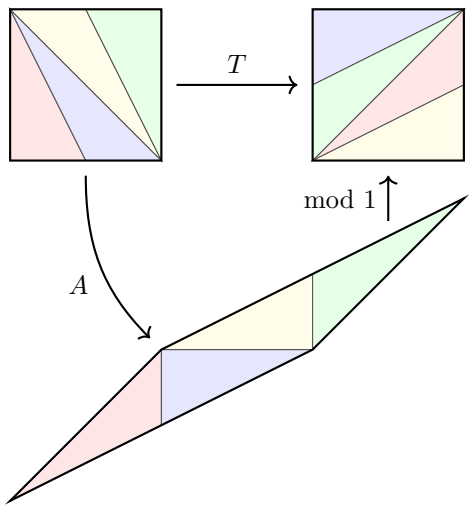
\begin{figure}[ht]
        \centering
        \begin{tikzpicture}[font=\footnotesize]
            \draw[draw=black!70,fill=red!10,very thin] (0,0) -- (1,0) -- (0,2) -- cycle;
            \draw[draw=black!70,fill=blue!10,very thin] (1,0) -- (2,0) -- (0,2) -- cycle;
            \draw[draw=black!70,fill=yellow!10,very thin] (0,2) -- (2,0) -- (1,2) -- cycle;
            \draw[draw=black!70,fill=green!10,very thin] (1,2) -- (2,0) -- (2,2) -- cycle;
            \draw[thick] (0,0) rectangle (2,2);
            
            \draw[thick,->,out=270,in=135] (1,-0.2) to node[below left] {$A$} (1.85,-2.35) ;

            \begin{scope}[shift={(0,-4.5)}]
                \draw[draw=black!70,fill=red!10,very thin] (0,0) -- (2,1) -- (2,2) -- cycle;
                \draw[draw=black!70,fill=blue!10,very thin] (2,1) -- (4,2) -- (2,2) -- cycle;
                \draw[draw=black!70,fill=yellow!10,very thin] (2,2) -- (4,2) -- (4,3) -- cycle;
                \draw[draw=black!70,fill=green!10,very thin] (4,3) -- (4,2) -- (6,4) -- cycle;
                \draw[thick] (0,0) -- (4,2) -- (6,4) -- (2,2) -- cycle;
            \end{scope}

            \draw[thick, ->] (2.2,1) -- node[above] {$T$} (3.8,1);
            \draw[thick, ->] (5,-0.8) -- node[left] {$\mod 1$} (5,-0.2);

            \begin{scope}[shift={(4,0)}]
                \draw[draw=black!70,fill=red!10,very thin] (0,0) -- (2,1) -- (2,2) -- cycle;
                \draw[draw=black!70,fill=blue!10,very thin] (0,1) -- (2,2) -- (0,2) -- cycle;
                \draw[draw=black!70,fill=yellow!10,very thin] (0,0) -- (2,0) -- (2,1) -- cycle;
                \draw[draw=black!70,fill=green!10,very thin] (0,1) -- (0,0) -- (2,2) -- cycle;
                \draw[thick] (0,0) rectangle (2,2);
            \end{scope}
        \end{tikzpicture}
        \caption{The (bijective) action of the cat map on the 2-torus}
        \label{fig:cat-map}
    \end{figure}
    A 3-dimensional example is provided, for instance, by
    \begin{equation*}
        A = \begin{bmatrix}
            2 & 0 & 1 \\
            0 & 1 & 1 \\
            1 & 1 & 2
        \end{bmatrix}
    \end{equation*}
    and examples in any higher dimension can be obtained as block-diagonal matrices composed of the previous two, just to illustrate that this class is not empty. 
\end{example}

There are many other classes of dynamical systems that admit uniform Young towers with exponential tails. Young proves the requisite properties for uniformly expanding maps (\cite[Theorem~$4'$]{young1998}), piecewise hyperbolic maps (\cite[Theorem~5]{young1998}), logistic maps (\cite[Theorem~7]{young1998}) and Hénon attractors (\cite[Theorem~8]{young1998}). We refer the reader to the source for the precise definitions of these various settings. 

\begin{remark}[on examples with the aperiodicity condition]\label{rem:aperiodicity} $ $\\
    The additional property of aperiodicity (see Definition~\ref{def:young-tower}) is required for Theorem~\ref{thm:mcdiarmid-young}, which is the basis of the proofs in the next section. It is equivalent to the mixing of $(\Delta,\hat\nu,\hat T)$ (see \cite[Section~6.1]{luzzatto2015}) and implies (but is not implied by) the mixing of $(M,\nu, T)$.

    It appears to be a \enquote{folk theorem} that if an Axiom A system is additionally topologically mixing, then the associated Young tower is aperiodic. This was confirmed in correspondence by experts on the field, but we could not find an explicit statement of this result in the literature.
\end{remark}

\section{Quantitative Estimates of the Generalization Error}\label{sec:rates}

\subsection{Assumptions and Approach}

We begin by requiring that $(Y_i)_{i\in\N}$ follows the distribution of an observable process (per Remark~\ref{rem:observable-processes}) over a suitable dynamical system.

\begin{assumption}[Young tower data]\label{ass:young-data} $ $\\
    In the setting of Assumption~\ref{ass:general-setting}, we additionally let $(M,T)$ be a dynamical system modeled by an aperiodic uniform Young tower with exponential tails and $\nu$ its mixing $T$-invariant distribution by Proposition~\ref{prop:young-tower-invariant-measure}. Then let $\widetilde Y_1 \sim \nu$ and $\widetilde Y_i= T^{i-1}(\widetilde Y_1)$ for all $i\in \N$. Finally, let $g:M\to [0,1]^d$ be $L$-Lipschitz continuous, $L>0$, and let $Y_i=g(\widetilde Y_i)$ for all $i\in\N$.
\end{assumption}

By construction, $(\widetilde Y_i)_{i\in\N}$ has the same distribution as the canonical process of the dynamical system. In particular, the process is mixing and thus ergodic. Therefore, the process of observables $(Y_i)_{i\in\N}$ is ergodic as well, with the invariant distribution $\mu=g_*\nu$.

\begin{remark}[on the applicability of Chazottes-Gouëzel]\label{rem:chazottes-gouezel-for-observables} $ $\\
    Theorem~\ref{thm:mcdiarmid-young} is stated for separately Lipschitz observables of the canonical process $(\widetilde Y_i)_{i\in\N}$. However, any separately Lipschitz observable $K$ of the observable process $(Y_i)_{i\in\N}$ can be re-stated as the separately Lipschitz observable $\widetilde K=K\circ (g,\dotsc,g)$ of the canonical process, with Lipschitz constants $\widetilde L_i = L_i \cdot L$. Consequently, Theorem~\ref{thm:mcdiarmid-young} is still applicable in our setting, but with the $(M,T)$-dependent constant $C$ replaced by $CL^2$ (which also depends on $g$).
\end{remark}

\begin{remark}[on learning the invariant distribution of the state space]\label{rem:learning-states} $ $\\
    The reason we consider observable processes $(Y_i)_{i\in\N}$ instead of the process of states $(\widetilde Y_i)_{i\in\N}$ itself is that the latter is $M$-valued while our generators and discriminators are defined on the unit cube $[0,1]^d$. Notably, since $M$ must be a Riemannian manifold, we can never have $M=[0,1]^d$.
    
    However, we may sometimes prefer to learn the invariant distribution $\nu$ of the process of states itself, rather than some push-forward $\mu=g_*\nu$. This can still be accomplished by suitable choice of observables. Consider the state space $M=\T^d$ (e.g. for hyperbolic torus automorphisms), whose elements we identify with their representatives in $[0,1]^d$ in the following, by a slight abuse of notation. Unlike the unit cube, it comes equipped with the metric $\d_{\T^d}$, which means that we need to take care when defining Lipschitz continuous observables.
    
    Let the invariant distribution $\nu$ be absolutely continuous w.r.t. $\lambda_{|\T^d}$, such that its density $f_{\nu}$ satisfies $f_{\nu} \in C^{k,\alpha}(\T^d,\R)$, $f_{\nu} \geq \kappa$ and $\norm{f_{\nu}}_{C^{k,\alpha}(\T^d,\R)} \leq M$ for some $\kappa,M>0$. This is analogous to Assumption~\ref{ass:true-measure-hoelder}, but notably the Hölder continuity of the $k$-th derivatives is also w.r.t. $\d_{\T^d}$.

    The goal is to find a Lipschitz continuous observable $g:(\T^d,\d_{\T^d})\to ([0,1]^d,\abs{\cdot})$ such that $\mu=g_*\nu$ satisfies the regular Assumption~\ref{ass:true-measure-hoelder}, i.e., its density $f_\mu$ lies in
    \begin{equation*}
        \cT \coloneq \left\{ f \in C^{k,\alpha}([0,1]^d,\R) \Mid f_{\mu} \geq \kappa, \ \norm{f_{\mu}}_{C^{k,\alpha}} \leq M\right\},
    \end{equation*}
    and $\mu$ is a good approximation of $\nu$. We cannot simply choose the identity, since the embedding of $\T^d$ into $[0,1]^d$ is not continuous (or even well-defined). Instead, for $\epsilon\in (0,\tfrac 12)$, define
    \begin{equation*}
        \psi : [0,1]\to [0,1], \qquad \psi(x) = \begin{cases}
            \frac{x}{1-\epsilon} & \text{if} \ 0\leq x \leq 1-\epsilon,\\
            \frac{1-x}{\epsilon} & \text{if} \ 1-\epsilon \leq x \leq 1.
        \end{cases}
    \end{equation*}
    as well as its component-wise application $g: [0,1]^d\to [0,1]^d, \ g(x_1,\dotsc,x_d)\mapsto (\psi(x_1),\dotsc,\psi(x_d))$. Since both $\psi$ and $g$ are 1-periodic, they are well-defined as functions on $\T$ and $\T^d$ respectively.
    
    Let $I_1 = [0,1-\epsilon]$, $I_2=[1-\epsilon,1]$, $s_1=\frac{1}{1-\epsilon}$ and $s_2=\frac{-1}{\epsilon}$. We divide the unit cube into $2^d$ rectangles $B_i = \bigtimes_{j=1}^d I_{i_j}$, $i\in \{1,2\}^d$. Then for any $i\in \{1,2\}^d$, $g_{|B_i}$ is an affine linear diffeomorphism $B_i \to [0,1]^d$ with constant Jacobian $Dg_{|B_i} \diag(s_{i_1},\dotsc,s_{i_d})$. The operator norm of this Jacobian is upper-bounded by $\frac{1}{\epsilon}$, so $g$ is $\frac 1 \epsilon$-Lipschitz continuous (even as a function $\T^d\to [0,1]^d$) by the mean value theorem.

    The local inverse $h_i=g_{|B_i}^{-1} : [0,1]^d \to B_i$ is also affine linear and its Jacobian is $\diag(s_{i_1}^{-1},\dotsc,s_{i_d}^{-1})$. In particular, $h_i$ is $(1-\epsilon)$-Lipschitz continuous. The transformation theorem yields
    \begin{equation*}
        \frac{d(g_*\nu_{|B_i})}{d\lambda}(y) = f_{\nu}(h_i(y)) \cdot (1-\epsilon)^{m(i)} \cdot \epsilon^{d-m(i)},
    \end{equation*}
    where $m(i)=\card{\{j\in\{1,\dotsc,d\} \mid i_j=1\}}$. We write $f_i = f_\nu \circ h_i$ and note that
    \begin{equation*}
        f_\mu = \frac{d(g_*\nu)}{d\lambda} = \sum_{i\in\{1,2\}^d} \frac{d(g_*\nu_{|B_i})}{d\lambda} = \sum_{i\in\{1,2\}^d} f_i \cdot (1-\epsilon)^{m(i)}\cdot \epsilon^{d-m(i)}
    \end{equation*}
    is a convex combination of these $f_i$. A simple inductive proof shows that
    \begin{equation*}
        D_\beta f_i(y) = \prod_{j=1}^d s_{i_j}^{-\beta_j} \cdot D_\beta f_{\nu}(h_i(y))
    \end{equation*}
    holds for all $i\in\{1,2\}^d$, $\beta\in\N_0^d$ and $y\in [0,1]^d$. This, together with the fact that $h_i$ is a contraction and $\d_{\T^d}(y,y')\leq \abs{y-y'}$, immediately shows that $f_i\in \cT$ and thus $f_\mu\in\cT$, since the defining inequalities are stable under convex combination. In other words, $g$ satisfies Assumption~\ref{ass:young-data} and $\mu$ satisfies Assumption~\ref{ass:true-measure-hoelder} - with the same constants as $\nu$.

    Finally, it remains to show that $\mu$ is a good approximation to $\nu$. To this end, let $y\in [0,1]^d$ be arbitrary, let $\ind=(1,\dotsc,1)\in \{1,2\}^d$ and note that
    \begin{align*}
        \abs{f_{\nu}(y)-f_\mu(y)} &\leq \sum_{i\in\{1,2\}^d} \abs{f_{\nu}(y)-f_{\nu}(h_i(y))} \cdot (1-\epsilon)^{m(i)}\cdot \epsilon^{d-m(i)}\\
        &\leq \abs{f_{\nu}(y)-f_{\nu}(h_\ind(y))} + (2^d-1) M\cdot\epsilon\\
        &\leq \sqrt{d} M \cdot \abs{y-(1-\epsilon)y} + (2^d-1) M\cdot\epsilon\\
        &\leq dM \cdot \epsilon + (2^d-1) M\cdot\epsilon\\
        &= (2^d+d-1)M\cdot \epsilon,
    \end{align*}
    where we used the $\sqrt{d}M$-Lipschitz continuity of $f_\nu$ (due to the mean value theorem) in the third line and $\diam([0,1]^d)=\sqrt{d}$ in the fourth. Consequently, we have
    \begin{equation*}
        \djs(\nu,\mu) \leq \frac{\ln(2)}{2} \dtv(\nu,\mu) \leq \frac{\ln(2)}{2}(2^d+d-1)M\cdot \epsilon \in \cO(\epsilon).
    \end{equation*}
    In other words, the invariant distribution of the process $Y_i = g(\widetilde Y_i)$ can be learned in our setting and provides an arbitrarily good (for $\epsilon\to 0$) approximation of $\mu$.
\end{remark}    

So far, we have required that the true measure density $f_\mu$, the generators $\phi$ and the discriminators $\xi$ are $k$ (respectively $k-1$) times $\alpha$-Hölder differentiable for some degree of differentiability $k\geq 2$ and Hölder-exponent $\alpha\in (0,1]$. While $k\geq 2$ suffices for the vanishing of the model errors in Theorem~\ref{thm:hypothesis-sufficiency}, we require a much higher degree of regularity to obtain rate estimates in this section, especially in practical applications to image data, where $d$ can be in the millions.

\begin{assumption}[high regularity]\label{ass:high-regularity} $ $\\
    Let Assumption~\ref{ass:model-hoelder} hold. We additionally require that the degree of differentiability $k$ in Assumptions~\ref{ass:true-measure-hoelder} and \ref{ass:model-hoelder} satisfies $k > 2 - \alpha + \frac d2$.
\end{assumption}

Under these new assumptions, we can prove rates for the almost sure convergence of $\egenmu(n)$ and $\egenlambda(n)$ to zero. We follow a similar approach, adapted from \cite[Chapter~4]{asatryan2023}, for both errors. Let $(\Theta,\rho)$ be a compact metric space. For each $n\in \N$, let $\X\ho n= (X_\theta\ho n)_{\theta\in \Theta}$ be a centered, path-wise continuous process and denote its absolute supremum by
\begin{equation*}
    \epsilon(n) = \sup_{\theta\in \Theta} \Abs{X_\theta\ho n}.
\end{equation*}
Then we use McDiarmid's inequality prove that $\X\ho n$ has subgaussian increments with respect to $\rho_n=n^{-\frac 12} \rho$ and apply Dudley's inequality to obtain
\begin{equation*}
    \E[\epsilon(n)] \leq \gamma_2 n^{-\frac 12}.
\end{equation*}
Another application McDiarmid's inequality, this time directly to $\epsilon(n)$, proves that
\begin{equation*}
    \P\bigl( \epsilon(n) - \E[\epsilon(n)] \geq t\bigr) \leq \exp\bigl(-\gamma_3 nt^2\bigr).
\end{equation*}
Finally, we use the Borel-Cantelli lemma to prove that
\begin{equation*}
    \epsilon(n) < \tau \sqrt{\frac{\log n}{n}}
\end{equation*}
for all but finitely many $n$, almost surely, for any sufficiently large $\tau$.

For $\egenlambda(n)$, we omit the details because they are completely analogous to the proof for $\egen(n)$ in \cite[Chapter~4]{asatryan2023}, using the regular McDiarmid's inequality (Theorem~\ref{thm:mcdiarmid-iid}).

\begin{theorem}[almost sure convergence to zero, with rates, independent generalization error]\label{thm:error-convergence-rate-lambda}
    Let Assumption~\ref{ass:high-regularity} hold and let $\tau > \sqrt{2 (\log B)^2}$. Then for $\P$-almost every $\omega\in \Omega$, there exists a natural number $N(\omega)\in\N$ such that
    \begin{equation*}
        \egenlambda(n) < \tau \sqrt{\frac{\log n}{n}}
    \end{equation*}
    holds for all $n\geq N(\omega)$.
\end{theorem}

The proof for $\egenmu(n)$, which we give in the following subsections, differs first and foremost in that Assumption~\ref{ass:young-data} forces us to use the aforementioned Chazottes-Gouëzel inequality, Theorem~\ref{thm:mcdiarmid-young}, which generalizes McDiarmid's inequality to Young towers. Since the observables $K$ there need to be separately Lipschitz instead of bounded, we choose a rescaled $\norm{\cdot}_{C^1}$-distance as $\rho$ instead of the $\norm{\cdot}_\infty$-distance, which then leads to different covering numbers as described in Proposition~\ref{prop:hoelder-entropy-C1}.

Note that this approach is not specific to Assumption~\ref{ass:young-data}. Any stochastic process $(Y_i)_{i\in\N}$ that satisfies a kind of McDiarmid's inequality -- like geometrically ergodic Markov chains \cite{dedecker2015,douc2018,havet2020} -- would do just as well, requiring only minor alterations to the proofs.

\subsection{Subgaussian Increments of the Empirical Process}

We begin by defining the empirical process $\X\ho n$. Let $X_\xi\ho n = \hat\cL^\mu(\xi,n) - \cL^\mu(\xi)$ for all $n\in\N$ and $\xi\in\chd$. The index space is $\Theta = \chd$ endowed with the metric
\begin{equation*}
    \rho(\xi,\xi') = \sqrt{C}L\left(\frac{C_1}{B^2}+\frac{\sqrt{d}}{B}\right) \norm{\xi-\xi'}_{C^1},
\end{equation*}
where $C$ is the positive constant associated to the dynamical system of Assumption~\ref{ass:young-data} by Theorem~\ref{thm:mcdiarmid-young}. Let us verify the initial conditions for the approach

\begin{proposition}[initial conditions]\label{prop:initial-conditions-mu} $ $\\
    Let Assumptions~\ref{ass:model-hoelder} and $\ref{ass:young-data}$ hold. Let $n\in\N$. Then
    \begin{enumerate}[label=(\roman*)]
        \item $(\Theta,\rho)$ is a compact metric space.
        \item $\X\ho n$ is a centered, path-wise continuous process.
    \end{enumerate}
\end{proposition}

Note that, for now, Assumption~\ref{ass:young-data} is only required because $\rho$ depends on it via $C$.

\begin{proof}
    \textbf{Assertion (i):} Since compactness transfers to weaker topologies, $\chd$ is compact w.r.t. $\norm{\cdot}_{C^1}$ by \cite[Proposition~3.11]{asatryan2023} (recall that $k\geq 2$, so the $C^{k-1,\alpha}$ topology is indeed stronger than $C^1$). Since $\Theta=\chd$ and $\rho$ is topologically equivalent to $\norm{\cdot}_{C^1}$, $(\Theta,\rho)$ is also compact.

    \textbf{Assertion (ii):} It is clear that each $X_\xi\ho n$ is centered. For the path-wise continuity, consider $Y_1,\dotsc,Y_n$ fixed. Then we have
    \begin{align*}
        \Abs{X_{\xi}\ho n - X_{\xi'}\ho n} &\leq \Abs{\hat \cL^\mu(\xi,n)-\hat \cL^\mu(\xi',n)} + \Abs{\cL^\mu(\xi) - \cL^\mu(\xi')}\\
        &\leq \frac{2}{B} \norm{\xi-\xi'}_{C^1} = \frac{2}{\sqrt{C}L\left(\frac{C_1}{B}+\sqrt{d}\right)} \rho(\xi,\xi')
    \end{align*}
    for all $\xi,\xi'\in\Theta$ by Proposition~\ref{prop:target-lipschitz-continuity}, i.e., the empirical process is even path-wise Lipschitz continuous.
\end{proof}

Note that $\rho_n = n^{-\frac 12}\rho$ is topologically equivalent to $\rho$, so we do not need to check compactness and path-wise continuity again for each $\rho_n$.

In order to prove that $\X\ho n$ has subgaussian increments, we need to define a separately Lipschitz observable $K$ for Theorem~\ref{thm:mcdiarmid-young}. Therefore, we need the following lemma.

\begin{lemma}[derivative bound for the helper function]\label{lem:derivative-bound-log-quotient} $ $\\
    Let Assumption~\ref{ass:model-hoelder} hold. Let $\xi,\xi' \in \chd$. Then the function
    \begin{equation*}
        f : [0,1]^d \to \R,\quad y\mapsto \log \frac{\xi(y)}{\xi'(y)}
    \end{equation*}
    is well-defined, continuously differentiable and satisfies
    \begin{equation*}
        \norm{Df}_\infty = \sup_{y\in [0,1]^d} \norm{Df(y)}_2 \leq \left(\frac{C_1}{B^2}+\frac{\sqrt{d}}{B}\right) \norm{\xi-\xi'}_{C^1}.
    \end{equation*}
\end{lemma}

\begin{proof}
    First, recall that $B\leq \xi,\xi' \leq 1-B$, so we have
    \begin{equation*}
        0 < \frac{B}{1-B} \leq \frac{\xi}{\xi'} \leq \frac{1-B}{B}
    \end{equation*}
    and the logarithm is well-defined and continuously differentiable on this interval. Therefore, $f$ is well-defined and continuously differentiable by the chain rule. Applying the chain rule and the quotient rule yields
    \begin{equation*}
        Df(y) = \frac{\xi'(y)}{\xi(y)} \cdot \frac{\xi'(y) D\xi(y) - \xi(y)D\xi'(y)}{\xi'(y)^2} = \frac{\xi'(y) D\xi(y) - \xi(y)D\xi'(y)}{\xi(y)\xi'(y)}
    \end{equation*}
    and thus 
    \begin{align*}
        \norm{Df(y)}_2 &= \Norm{\frac{\xi'(y) D\xi(y) - \xi(y)D\xi'(y)}{\xi(y)\xi'(y)}}_2\\
        &= \Norm{\frac{\bigl(\xi'(y)-\xi(y)\bigr) D\xi(y) + \xi(y)\bigl(D\xi(y)-D\xi'(y)\bigr)}{\xi(y)\xi'(y)}}_2\\
        &\leq \Norm{\frac{\bigl(\xi'(y)-\xi(y)\bigr) D\xi(y)}{\xi(y)\xi'(y)}}_2 + \Norm{\frac{D\xi(y)-D\xi'(y)}{\xi'(y)}}_2 \\
        &\leq \frac{\norm{D\xi(y)}_2}{B^2}\abs{\xi'(y)-\xi(y)}  + \frac{1}{B} \norm{D\xi(y)-D\xi'(y)}_2\\
        &\leq \frac{C_1}{B^2} \norm{\xi-\xi'}_\infty + \frac{1}{B} \norm{D(\xi-\xi')}_\infty\\
        &\leq \frac{C_1}{B^2} \norm{\xi-\xi'}_{C^1} + \frac{\sqrt{d}}{B} \norm{\xi-\xi'}_{C^1}
    \end{align*}
    for all $y\in [0,1]^d$. The same bound thus holds for $\norm{Df}_\infty$.
\end{proof}

\begin{proposition}[subgaussian increments of the empirical process]\label{prop:subgaussian-increments-mu} $ $\\
    Let Assumptions~\ref{ass:model-hoelder} and $\ref{ass:young-data}$ hold. Let $n\in\N$. Then $\X\ho n$ has subgaussian increments with respect to $\rho_n = n^{-\frac 12} \rho$.
\end{proposition}

\begin{proof}
    $\X\ho n$ is a centered process by Proposition~\ref{prop:initial-conditions-mu}. For $\xi,\xi'\in\Theta$, we define the observable
    \begin{equation*}
        K : \bigl([0,1]^d\bigr)^n \to \R, \quad (y_1,\dotsc,y_n)\mapsto \frac 1n \sum_{i=1}^n \bigl[\log\xi(y_i) - \log\xi'(y_i)\bigr].
    \end{equation*}
    Let $f$ be the helper function from Lemma~\ref{lem:derivative-bound-log-quotient}. Then we have
    \begin{equation*}
       K(y_1,\dotsc,y_n) = \frac 1n \sum_{i=1}^n f(y_i)
    \end{equation*}
    and thus
    \begin{align*}
        \abs{K(y_1,\dotsc,y_i,\dotsc,y_n) -  K(y_1,\dotsc,y_i',\dotsc,y_n)} &= \frac 1n \abs{f(y_i)-f(y_i')} \leq \frac 1n \norm{Df}_\infty \cdot \abs{y_i-y_i'}
    \end{align*}
    by the mean value theorem, for all $y_1,\dotsc,y_n,y_1',\dotsc,y_n'\in [0,1]^d$ and $i\in \otn$. So by Lemma~\ref{lem:derivative-bound-log-quotient}, $K$ is a separately Lipschitz observable with coefficients
    \begin{equation*}
        L_i = \frac 1n\left(\frac{C_1}{B^2}+\frac{\sqrt{d}}{B}\right) \norm{\xi-\xi'}_{C^1}, \quad i\in\otn.
    \end{equation*}
    By Remark~\ref{rem:chazottes-gouezel-for-observables}, we can apply Theorem~\ref{thm:mcdiarmid-young} to $Y_1,\dotsc,Y_n$ to obtain
    \begin{equation*}
        \vpnorm{K(Y_1,\dotsc,Y_n)}^2 \leq \frac{CL^2}{n}\left(\frac{C_1}{B^2}+\frac{\sqrt{d}}{B}\right)^2 \norm{\xi-\xi'}_{C^1}^2 = \rho_n(\xi,\xi')^2.
    \end{equation*}
    Since the variance proxy is invariant under constant summands and
    \begin{equation*}
        X_\xi\ho n- X_{\xi'}\ho n = K(Y_1,\dotsc,Y_n)-\E[K(Y_1,\dotsc,Y_n)],
    \end{equation*}
    $\X\ho n$ has subgaussian increments w.r.t. $\rho_n$.
\end{proof}

\subsection{Bounding the Expected Value}

Since the empirical process is centered, path-wise continuous on a compact index space $(\Theta,\rho_n)$ and has subgaussian increments w.r.t. the same, we can apply Dudley's inequality. To this end, let us first estimate the metric entropy of $(\Theta,\rho)$.

\begin{lemma}[metric entropy of the index space]\label{lem:metric-entropy-mu} $ $\\
    Let Assumption~\ref{ass:model-hoelder} hold. For all $\epsilon>0$, we have
    \begin{equation*}
        \log N(\Theta,\rho,\epsilon) \leq \hat\gamma_1 \epsilon^{-\frac{d}{\alpha+k-2}},
    \end{equation*}
    where $\hat\gamma_1=\hat\gamma_1(d,\alpha+k,C,L,B,C_1,C_2)$ is a positive constant.
\end{lemma}

\begin{proof}
    Let $\epsilon>0$ be arbitrary. We have
    \begin{equation*}
        \Theta = \chd \subseteq B_{C^{k-1,\alpha}([0,1]^d,\R)}(0,C_2).
    \end{equation*}
    By Proposition~\ref{prop:covering-number-monotonicity} that implies
    \begin{align*}
        N\left(\Theta,\norm{\cdot}_{C^1},\epsilon\right) &\leq N\left(B_{C^{k-1,\alpha}([0,1]^d,\R)}(0,C_2),\norm{\cdot}_{C^1},\frac \epsilon 2\right)\\
        &= N\left(B_{C^{k-1,\alpha}([0,1]^d,\R)}(0,1),\norm{\cdot}_{C^1},\frac \epsilon {2 C_2}\right)
    \end{align*}
    and using Proposition~\ref{prop:hoelder-entropy-C1} (applied to $k'=k-2\geq 0$) we can estimate
    \begin{equation*}
        \log N\left(\Theta,\norm{\cdot}_{C^1},\epsilon\right) \leq \hat \gamma(d,1,\alpha+k-2) \left(\frac{\epsilon}{2C_2}\right)^{-\frac{d}{\alpha+k-2}}.
    \end{equation*}
    Since $\rho$ is just a multiple of (the metric induced by) $\norm{\cdot}_{C^1}$, we conclude
    \begin{equation*}
        \log N\left(\Theta,\rho,\epsilon\right) \leq \hat \gamma(d,1,\alpha+k-2) \left(\frac{\epsilon}{2C_2\sqrt{C}L\left(\frac{C_1}{B^2}+\frac{\sqrt{d}}{B}\right)}\right)^{-\frac{d}{\alpha+k-2}}.
    \end{equation*}
    The assertion is fulfilled with the constant
    \begin{equation*}
        \hat\gamma_1 = \hat \gamma(d,1,\alpha+k-2) \left(2C_2\sqrt{C}L\left(\frac{C_1}{B^2}+\frac{\sqrt{d}}{B}\right)\right)^{\frac{d}{\alpha+k-2}}. \tag*{\qedhere}
    \end{equation*}
\end{proof}

\begin{proposition}[bounding the expected value]\label{prop:dudley-bound-mu} $ $\\
    Let Assumptions~\ref{ass:young-data} and \ref{ass:high-regularity} hold. Let $n\in\N$. Then we have
    \begin{equation*}
        \E[\egenmu(n)] \leq \frac{\hat\gamma_2}{\sqrt{n}},
    \end{equation*}
    where $\hat\gamma_2=\hat\gamma_2(d,\alpha+k,C,L,B,C_1,C_2)$ is a positive constant.
\end{proposition}

\begin{proof}
    Let us consider the constant discriminator $\overline \xi \equiv \frac 12$. By the choice of coefficients in Assumption~\ref{ass:model-hoelder}, in particular, we have $C_2 \geq 2^{2^6} > \frac 12$. Thus, $\overline \xi \in \chd$. This discriminator is of interest because $X_{\overline\xi}\ho n \equiv 0$, which implies $\sup_{\xi\in\Theta} X_\xi\ho n \geq X_{\overline \xi}\ho n = 0$ and thus
    \begin{equation*}
        \sup_{\xi\in\Theta} X_\xi\ho n = \sup_{\xi\in\Theta} \psi_+\left(X_\xi\ho n\right),
    \end{equation*}
    where $\psi_+(x)=\max\{x,0\}$. Analogously,
    \begin{equation*}
        \sup_{\xi\in\Theta} \left(-X_\xi\ho n\right) = \sup_{\xi\in\Theta} \psi_+\left(-X_\xi\ho n\right) = \sup_{\xi\in\Theta} \psi_-\left(X_\xi\ho n\right),
    \end{equation*}
    where $\psi_-(x)=\max\{-x,0\}$. Since $\abs{x}=\psi_+(x)+\psi_-(x)$, we have
    \begin{align*}
        \egenmu(n) &= \sup_{\xi\in\Theta} \Abs{X_\xi\ho n} = \sup_{\xi\in\Theta} \psi_+\left(X_\xi\ho n\right) + \psi_-\left(X_\xi\ho n\right) \\
        &\leq \sup_{\xi\in\Theta} \psi_+\left(X_\xi\ho n\right) + \sup_{\xi\in\Theta} \psi_-\left(X_\xi\ho n\right)\\
        &= \sup_{\xi\in\Theta} X_\xi\ho n + \sup_{\xi\in\Theta} \left(-X_\xi\ho n\right).
    \end{align*}
    By Proposition~\ref{prop:initial-conditions-mu}, $(\Theta,\rho_n)$ is a compact metric space and both processes are centered and path-wise continuous. By Proposition~\ref{prop:subgaussian-increments-mu}, both have subgaussian increments with respect to $\rho_n$ (for $-\X\ho n$ this follows from the absolute homogeneity of $\vpnorm{\cdot}$). Applying Theorem~\ref{thm:dudley-integral} to both $\X\ho n$ and $-\X\ho n$, followed by a change of variables, yields
    \begin{align*}
        \E[\egenmu(n)] &\leq 24 \int_0^\infty \sqrt{\log N(\Theta,\rho_n,\epsilon)} \, d\epsilon = \frac{24}{\sqrt{n}} \int_0^\infty \sqrt{\log N(\Theta,\rho, \epsilon)} \, d\epsilon.
    \end{align*}
    Let $\delta=\diam(\Theta)$ with respect to $\rho$. Then we have
    \begin{equation*}
        \frac{24}{\sqrt{n}} \int_0^\infty \sqrt{\log N(\Theta,\rho, \epsilon)} \, d\epsilon = \frac{24}{\sqrt{n}} \int_0^\delta \sqrt{\log N(\Theta,\rho, \epsilon)} \, d\epsilon
    \end{equation*}
    and inserting Lemma~\ref{lem:metric-entropy-mu} on the right-hand side yields
    \begin{equation*} 
        \E[\egenmu(n)] \leq \frac{24}{\sqrt{n}} \int_0^\delta \sqrt{\hat\gamma_1 \epsilon^{-\frac{d}{\alpha+k-2}}} \, d\epsilon = \frac{24\sqrt{\hat\gamma_1}}{\sqrt{n}} \int_0^\delta \epsilon^{-\frac{d}{2(\alpha+k-2)}} \, d\epsilon.
    \end{equation*}
    The integral is finite near zero because of Assumption~\ref{ass:high-regularity} and we obtain
    \begin{equation*}
        \E[\egenmu(n)] \leq \frac{24\sqrt{\hat\gamma_1}}{\sqrt{n}} \cdot \frac{1}{1-\frac{d}{2(\alpha+k-2)}} \delta^{1-\frac{d}{2(\alpha+k-2)}}.
    \end{equation*}
    The assertion holds with the constant
    \begin{equation*}
        \hat\gamma_2 = 24\sqrt{\hat\gamma_1} \cdot \frac{1}{1-\frac{d}{2(\alpha+k-2)}} \delta^{1-\frac{d}{2(\alpha+k-2)}}. \tag*{\qedhere}
    \end{equation*}
\end{proof}

\subsection{Concentration Around the Expected Value}

\begin{proposition}[concentration around the mean]\label{prop:concentration-mu} $ $\\
    Let Assumptions~\ref{ass:young-data} and \ref{ass:high-regularity} hold. Let $n\in\N$. Then we have
    \begin{equation*}
        \P\bigl(\egenmu(n)-\E[\egenmu(n)]\geq t\bigr) \leq \exp\bigl(-\hat\gamma_3 n t^2\bigr) \quad \forall t\geq 0,
    \end{equation*}
    where $\hat\gamma_3=\hat\gamma_3(C,L,B,C_1)$ is a positive constant.
\end{proposition}

\begin{proof}
    We define the observable
    \begin{equation*}
        K : \bigl([0,1]^d\bigr)^n \to \R, \quad (y_1,\dotsc,y_n)\mapsto \sup_{\xi\in\Theta} \Abs{f_{\xi}(y_1,\dotsc,y_n)},
    \end{equation*}
    where
    \begin{equation*}
        f_{\xi}(y_1,\dotsc,y_n) = \frac 1n \sum_{i=1}^n \log\xi(y_i) - \cL^\mu(\xi)
    \end{equation*}
    for all $\xi\in\Theta$ and $y_1,\dotsc,y_n\in [0,1]^d$. Then we have
    \begin{equation*}
        K(Y_1,\dotsc,Y_n) = \egenmu(n).
    \end{equation*}
    Additionally, since $\log$ is $\frac 1B$-Lipschitz continuous on $[B,1-B]$ and each $\xi\in\Theta$ is $C_1$-Lipschitz continuous by the mean value theorem, we have
    \begin{align*}
        K&(y_1,\dotsc,y_i',\dotsc,y_n) = \sup_{\xi\in\Theta} \Abs{f_\xi(y_1,\dotsc,y_i',\dotsc,y_n)}\\
        &= \sup_{\xi\in\Theta} \Abs{f_\xi(y_1,\dotsc,y_i,\dotsc,y_n)+ \frac 1n \log\xi(y_i') - \frac 1n \log\xi(y_i)}\\
        &\leq K(y_1,\dotsc,y_i,\dotsc,y_n) +  \frac 1n \sup_{\xi\in\Theta} \babs{\log\xi(y_i') - \log\xi(y_i)}\\
        &\leq K(y_1,\dotsc,y_i,\dotsc,y_n) + \frac{C_1}{Bn} \abs{y_i'-y_i}
    \end{align*}
    for all $y_1,\dotsc,y_n,y_1',\dotsc,y_n'\in [0,1]^d$ and $i\in \otn$. That makes $K$ a separately Lipschitz observable with coefficients
    \begin{equation*}
        L_i = \frac{C_1}{Bn}, \quad i\in\otn.
    \end{equation*}
    By Remark~\ref{rem:chazottes-gouezel-for-observables}, we can apply Theorem~\ref{thm:mcdiarmid-young} and Proposition~\ref{prop:subgaussian-tails-onesided} to $Y_1,\dotsc,Y_n$ to obtain the assertion with the constant
    \begin{equation*}
        \hat \gamma_3 = \frac{B^2}{2CL^2C_1^2}. \tag*{\qedhere}
    \end{equation*}
\end{proof}

Using this concentration around the mean, we conclude by a Borel-Cantelli argument.

\begin{theorem}[almost sure convergence to zero, with rates]\label{thm:error-convergence-rate-mu} $ $\\
    Let Assumptions~\ref{ass:young-data} and \ref{ass:high-regularity} hold and let $\tau > \frac{1}{\sqrt{\hat\gamma_3}}$. Then for $\P$-almost every $\omega\in \Omega$, there exists a natural number $\hat N(\omega)\in\N$ such that
    \begin{equation*}
        \egenmu(n) < \tau \sqrt{\frac{\log n}{n}}
    \end{equation*}
    holds for all $n\geq \hat N(\omega)$.
\end{theorem}

\begin{proof}
    Let $\beta > 1$ be fixed. By Proposition~\ref{prop:dudley-bound-mu} and Proposition~\ref{prop:concentration-mu}, we have
    \begin{align*}
        \P\left(\egenmu(n) \geq \frac{\hat\gamma_2}{\sqrt{n}} + \beta \sqrt{\frac{\log n}{\hat\gamma_3 n}}\right) &\leq \P\left(\egenmu(n)-\E[\egenmu(n)] \geq \beta \sqrt{\frac{\log n}{\hat\gamma_3 n}}\right) \\
        &\leq \exp\left(-\hat\gamma_3 n \beta^2 \frac{\log n}{\hat\gamma_3 n}\right) = n^{-\beta^2}
    \end{align*}
    and since $\beta>1$, the corresponding series is dominated by the absolutely convergent series $\sum_{n\in\N} n^{-\beta^2}$. By the Borel-Cantelli lemma, this implies
    \begin{equation}\label{heq:14031527}
        \P\left(\egenmu(n) \geq \frac{\hat\gamma_2}{\sqrt{n}} + \beta \sqrt{\frac{\log n}{\hat\gamma_3 n}} \ \text{infinitely often}\right) = 0.
    \end{equation}
    Let $\omega\in \Omega$ (the underlying probability space) be fixed and let the complementary event hold. Then there exists some $N'(\omega,\beta)\in\N$ such that
    \begin{equation*}
        \egenmu(n) < \frac{\hat\gamma_2}{\sqrt{n}} + \beta \sqrt{\frac{\log n}{\hat\gamma_3 n}}
    \end{equation*}
    holds for all $n\geq N'(\omega,\beta)$.
    
    Now let $\tau> \frac{1}{\sqrt{\hat\gamma_3}}$ be arbitrary and choose $\beta = \frac{1+\tau\sqrt{\hat\gamma_3}}{2}>1$ such that $ \tau > \frac{\beta}{\sqrt{\hat\gamma_3}}$ and thus $\sigma=\tau - \frac{\beta}{\sqrt{\hat\gamma_3}} > 0$. There exists $N''\in \N$ such that $\hat\gamma_2 \leq \sigma\sqrt{\log n}$ for all $n\geq N''$. If we fix $\omega\in\Omega$ in the complement of \eqref{heq:14031527} again, we obtain
    \begin{equation*}
        \egenmu(n) < \frac{\hat\gamma_2}{\sqrt{n}} + \beta \sqrt{\frac{\log n}{\hat\gamma_3 n}} \leq \sigma \sqrt{\frac{\log n}{n}} + \frac{\beta}{\sqrt{\hat\gamma_3}} \sqrt{\frac{\log n}{n}} = \tau \sqrt{\frac{\log n}{n}}
    \end{equation*}
    for all $n\geq \hat N(\omega)=\max\{N'(\omega,\beta),N''\}$. Note that while $N'(\omega,\beta)$ depends on $\beta$, $\beta$ is now determined by $\tau$ and $\hat\gamma_3$, so $N(\omega)$ depends only on $\omega$ and all the same parameters as $\hat\gamma_3$ (which we leave out of the notation). 
\end{proof}

\subsection{Proof of Almost Sure Convergence with Rates}

Finally, we come to the main result of this paper.

\begin{theorem}[consistency of GAL from young towers, with rates]\label{thm:consistency-with-rates} $ $\\
    Let Assumptions~\ref{ass:young-data} and \ref{ass:high-regularity} hold. For every $n\in\N$ fixed, there exists a solution $\hat\phi_n$ to the empirical minimax problem \eqref{eq:empirical-minimax-problem}. Let $\tau > \frac{1}{\sqrt{\gamma_3}} + \frac{1}{\sqrt{\hat\gamma_3}}$. Then for $\P$-almost every $\omega\in\Omega$, there exists a natural number $\widetilde N(\omega)\in\N$ such that
    \begin{equation*}
        \djs(\mu,\mu_{\hat\phi_n}) < \tau \sqrt{\frac{\log n}{n}}  \qquad \forall n\geq \widetilde N(\omega)
    \end{equation*}
\end{theorem}

\begin{proof}
    The existence of the empirical risk minimizers $\hat\phi_n$ is guaranteed by Theorem~\ref{thm:erm-existence}. By Proposition~\ref{prop:error-decomposition} we have $\djs(\mu,\mu_{\hat\phi_n}) \leq \egenmu(n) + \egenlambda(n)$. Choose $\tau',\tau''\in \R$ with $\tau=\tau' +\tau''$, $\tau' > \frac{1}{\sqrt{\gamma_3}}$ and $\tau'' > \frac{1}{\sqrt{\hat\gamma_3}}$, so we can apply Theorem~\ref{thm:error-convergence-rate-lambda} and Theorem~\ref{thm:error-convergence-rate-mu} respectively. As the union of two nullsets is a nullset, the two $\P$-almost sure events hold simultaneously for $\P$-almost every $\omega\in\Omega$. In that case, let $\widetilde N(\omega) = \max\{N(\omega),\hat N(\omega)\}$. We have
    \begin{equation*}
        \djs(\mu,\mu_{\hat\phi_n}) < \tau' \sqrt{\frac{\log n}{n}} + \tau'' \sqrt{\frac{\log n}{n}} = \tau \sqrt{\frac{\log n}{n}} \qquad \forall n\geq \widetilde N(\omega). \tag*{\qedhere}
    \end{equation*}
\end{proof}

\section{Conclusions and Outlook}\label{sec:conclusion}

In this work, we have proven that the traditional assumption of independence is not critical to the convergence of generative adversarial learning, explaining -- in part -- the empirical success of GANs trained on time series of chaotic dynamical systems, and opening the way for a better theoretical understanding of recent developments in the field of physical AI.  Most importantly, we were able to replace Hoeffding's / McDiarmid's inequality with Chazottes and Gouëzel's concentration inequality \ref{thm:mcdiarmid-young} to obtain the same $\cO(n^{-\frac 12})$ (up to a polylogarithm) rate of almost sure convergence. To adapt the Dudley estimates in \cite{asatryan2023}, we also proved to our knowledge novel metric entropy bounds for the embedding of spaces of Hölder differentiable functions into $C^1$ (as opposed to $C$).

The results in this paper concern generative adversarial learning for systems that admit an aperiodic uniform Young tower with exponential tails, but the methodology is largely not specific to this setting. GAL can easily be replaced by traditional maximum likelihood-based ERM learning, so long as McDiarmid's inequality is used to prove convergence in the i.i.d. case.

Likewise, any ergodic system that satisfies an appropriate exponential concentration inequality can be considered in place of Young towers. Geometrically ergodic Markov chains, for example, satisfy a type of McDiarmid's inequality \cite{dedecker2015,douc2018,havet2020}. One may also consider continuous time dynamical systems as in \cite{drygala2022}. In this case, it might be possible to access a continuous time analog of McDiarmid's inequality through the \emph{exponential decay of correlations} property discussed, for example, in \cite{gaposhkin1982}.

Another interesting direction for future research is replacing the infinite-dimensional hypothesis spaces of \cite{asatryan2023} by spaces of neural networks. \cite{belomestny2023} proves that ReQU networks approximate Hölder differentiable functions in $C^k$ norm. In this case, the model error is once again in consideration and must be balanced with the generalization error. To preserve the diffeomorphism of the generators, NeuralODE models, where the actual generator is the flow endpoint associated to a vector field modeled by the neural network, seem promising. Their use for the highly related (and aforementioned) problem of maximum likelihood-based ERM learning has been studied, e.g., in \cite{ehrhardt2025,marzouk2024}.

Finally, on the topic of Young towers, it is possible to sidestep the observables of Remark~\ref{rem:learning-states} fully by learning directly on the torus. This requires generators that are diffeomorphisms on the torus -- the Rosenblatt transformation is not, due to its lack of periodicity. Implementing such functions on the torus itself via neural networks can also lead to learning the dynamics directly. 

\appendix 

\section{On the Definition of Young Towers}\label{ap:young}

The content of this section summarizes \cite[Sections~1.1 and 1.2]{young1998}.

As in Section~\ref{sec:young-towers}, we require that $T$ is a $C^{1+\epsilon}$-diffeomorphism. If $W\subseteq M$ is a submanifold, let $\mu_W$ denote the measure on $W$ induced by the restriction of the Riemannian structure to $W$.

\begin{definition}[(un)stable disk] $ $\\
    An embedded disk $\gamma \subseteq M$, i.e., an embedded submanifold parameterized over the unit disk in some $\R^n$, is called an {unstable disk} if $\d(T^{-n}(x),T^{-n}(y))\to 0$ exponentially fast as $n\to\infty$, for all $x,y\in \gamma$. It is called a {stable disk} if $\d(T^{n}(x),T^{n}(y))\to 0$ exponentially fast as $n\to\infty$, for all $x,y\in \gamma$.
\end{definition}

The (un)stable disks define directions of contraction and expansion.

\begin{definition}[continuous family of $C^1$-(un)stable disks] $ $\\
    We say that $\Gamma^u = \{\gamma^u\}$ is a \emph{continuous family of $C^1$-unstable disks} if
    \vspace{-1.5mm}
    \begin{itemize}
        \item $K^s$ is an arbitrary compact set; $D^u$ is the unit disk of some $\R^n$;
        \item $\Phi^u : K^s \times D^u \to M$ is a map such that
        \begin{itemize}
            \item $\Phi^u$ maps $K^s\times D^u$ homeomorphically onto its image;
            \item $x\mapsto \Phi^u|_{\{x\} \times D^u}$ is a continuous map from $K^s$ into the space of $C^1$-embeddings of $D^u$ into $M$;
            \item $\gamma^u = \Phi^u(\{x\}\times D^u)$ is an unstable disk for each $x\in K^s$.
        \end{itemize}
    \end{itemize}
    \emph{Continuous families of $C^1$-stable disks} are defined analogously.
\end{definition}
\vspace{-4mm}
\begin{definition}[hyperbolic product structure]\label{def:hyperbolic-product-structure} $ $\\
    We say that $\Lambda \subseteq M$ has a \emph{hyperbolic product structure} if there exist continuous families $\Gamma^u=\{\gamma^u\}$ of $C^1$-unstable disks and $\Gamma^s=\{\gamma^s\}$ of $C^1$-stable disks with
    \vspace{-1.5mm}
    \begin{itemize}
        \item $\dim \gamma^u + \dim \gamma^s = \dim M$;
        \item Each $\gamma^u\in \Gamma^u$ meets each $\gamma^s\in \Gamma^s$ in exactly one point;
        \item The $\gamma^u$ are transversal to the $\gamma^s$, with the angles between them bounded away from zero;
        \item $\Lambda = \bigl(\bigcup \Gamma^u\bigr) \cap \bigl(\bigcup \Gamma^s\bigr)$.
    \end{itemize}
\end{definition}

This lets us express the first requirement of our system.

\textbf{(P1)} \textit{There exists a subset $\Lambda \subseteq M$ with a hyperbolic product structure and with $\mu_{\gamma^u}(\gamma^u\cap \Lambda)>0$ for every $\gamma^u\in \Gamma^u$.}

A hyperbolic product structure defines a kind of \enquote{coordinates} on $\Lambda$, since each point $x\in \Lambda$ must be the unique intersection of a stable and unstable disk. We refer to these disks as $\gamma^s(x)$ and $\gamma^u(x)$. We additionally define certain subsets of $\Lambda$ that take the shape of \enquote{rectangles} in these coordinates.

\begin{definition}[$s$- and $u$-subsets]\label{def:s-u-subsets} $ $\\
    Let $\Lambda\subseteq M$ have a hyperbolic product structure with the families $\Gamma^u$ and $\Gamma^s$. A subset $\Lambda_1 \subseteq \Lambda$ is called an \emph{$s$-subset} if $\Lambda_1$ also has a hyperbolic product structure and the corresponding families can be chosen as $\Gamma^u$ and $\Gamma^s_1$, where $\Gamma^s_1 \subseteq \Gamma^s$.

    \emph{$u$-subsets} are defined analogously.
\end{definition}

More precisely, an $s$-subset is a rectangle that extends fully in the stable direction and a $u$-subset extends fully in the unstable direction. The next requirement is that $\Lambda$ can be \emph{almost} partitioned into $s$-subsets $\Lambda_i$ which make a full return to $\Lambda$ at a return time $R_i$, and cross $\Lambda$ fully in the unstable direction when they return.

\textbf{(P2)} \textit{There exist pairwise disjoint $s$-subsets $(\Lambda_i)_{i\in I}$ of $\Lambda$ such that
\begin{itemize}
    \item Either $I=\{1,\dotsc,n\}$ for some $n\in\N$ or $I=\N$;
    \item For each $\gamma^u\in \Gamma^u$, we have $\mu_{\gamma^u}((\Lambda \setminus \bigcup_{i\in I} \Lambda_i) \cap \gamma^u) = 0$ ;
    \item For each $i\in I$, there exists $R_i \in \N$ such that $T^{R_i} (\Lambda_i)$ is a $u$-subset of $\Lambda$;
    \item For all $x \in \Lambda_i$, $T^{R_i}(\gamma^s(x)) \subseteq \gamma^s(T^{R_i}(x))$ and $T^{R_i}(\gamma^u(x)) \supset \gamma^u(T^{R_i}(x))$;
    \item For each $n\in\N$, there are at most finitely many $i\in I$ with $R_i = n$;
    \item $\min_{i\in I} R_i \geq R_0 > 1$, where $R_0$ depends on $T$.
\end{itemize}}

We define some notion of \emph{separation time} on $\Lambda$. Since $T\in C^{1+\epsilon}$, one can expect that two points that start very close to one another also have orbits very close to one another for some time, before they eventually separate (or converge to one another, if the points lie on the same stable disk). The exact definition of the separation time depends on the system in question and indeed Young defines several very different separation times in \cite[Sections~6-10]{young1998}. The separation time must satisfy the following requirements.

\textbf{(S)} \textit{For every pair $x, y \in \Lambda$, define the separation time $s_0(x, y)$ such that
\begin{enumerate}
    \item[(i)] $s_0(x, y) \in \N_0\cup\{\infty\}$ and depends only on $\gamma^s(x)$ and $\gamma^s(y)$;
    \item[(ii)] For each $n\in\N$, there does not exist an infinite subset $S\subseteq \Lambda$ such that $s_0(x,y)<n$ holds for all $x,y\in S$;
    \item[(iii)] For $x, y \in \Lambda_i$, $s_0(x, y) \geq R_i + s_0(T^{R_i}(x), T^{R_i}(y))$; in particular, $s_0(x, y) \geq R_i$;
    \item[(iv)] For $x \in \Lambda_i$, $y \in \Lambda_j$ with $i \neq j$ but $R_i = R_j$, we have $s_0(x, y) < R_i - 1$.
\end{enumerate}}

Finally, we require some quantitative estimates, on the expansion and contraction along the (un)stable disks, as well as the distortion of the Riemannian volume under $T$. Assume that, for some $C>0$ and $\beta \in (0,1)$, we have

\textbf{(P3)} \textit{For $y \in \gamma^s(x)$ and $n\in \N_0$, we have $\d(T^n(x), T^n(y)) \leq C \beta^n$.}

\textbf{(P4)} \textit{For $y \in \gamma^u(x)$ and $0 \leq k \leq n < s_0(x,y)$, we have
\begin{enumerate}[label=(\roman*)]
    \item $\d(T^n(x), T^n(y)) \leq C \beta^{s_0(x,y)-n}$;
    \item \begin{equation*}
        \log \prod_{i=k}^{n}  \frac{\det D T^u[T^i(x)]}{\det D T^u[T^i(y)]}  \leq C \beta^{s_0(x,y)-n}.
    \end{equation*}
\end{enumerate}}

\textbf{(P5)} \textit{
\begin{enumerate}[label=(\roman*)]
    \item For $y \in \gamma^s(x)$ and $n\in\N_0$, we have
    \begin{equation*}
        \log \prod_{i=n}^{\infty}  \frac{\det D T^u[T^i(x)]}{\det D T^u[T^i(y)]}  \leq C \beta^{n}.
    \end{equation*}
    \item For $\gamma, \gamma' \in \Gamma^u$, the projection $\Theta: \gamma \cap \Lambda \to \gamma' \cap \Lambda$ defined by $\Theta(x) = \gamma^s(x) \cap \gamma'$ is absolutely continuous. Its inverse induces a measure on $\gamma\cap \Lambda$ that is absolutely continuous w.r.t. $\mu_\gamma$, with the density given by
    \[
    \frac{d(\Theta^{-1}_* \mu_{\gamma'})}{d \mu_{\gamma}}(x) = \prod_{i=0}^{\infty} \frac{\det D T^u[T^i(x)]}{\det D T^u[T^i(\Theta(x))]}.
    \]
\end{enumerate}}

The system $(M,T)$ is modeled by a uniform Young tower if it fulfills the requirements \textbf{(S)} and \textbf{(P1)} through \textbf{(P5)}. The tower can then be defined as in Section~\ref{sec:young-towers}, following \cite[Section~1.3]{young1998}. Note that we follow Young in making the slight simplifying assumption that $\Lambda = \bigcup_{i\in I} \Lambda_i$, which actually only holds true up to $\mu$-nullsets.

\section{Subgaussian Concentration}\label{ap:subgaussians}

Subgaussian random variables are those whose moment-generating function is dominated by that of a Gaussian random variable, hence the name.

\begin{definition}[subgaussian random variable; variance proxy]\label{def:subgaussian} $ $\\
    A random variable is called \emph{subgaussian} with \emph{variance proxy} $\sigma^2\geq 0$ if
    \begin{equation*}
        \E[\exp(\tau (X-\E[X]))] \leq \exp\left(\frac{\tau^2 \sigma^2}{2}\right)
    \end{equation*}
    for all $\tau \in \R$. We further define
    \begin{equation*}
        \vpnorm{X} = \min \{ \sigma \geq 0 \mid X \ \text{is subgaussian with variance proxy} \ \sigma^2 \},
    \end{equation*}
    such that $\vpnorm{X}^2$ is the \emph{minimal variance proxy} for $X$.
\end{definition}

The most important property of subgaussian random variables in the context of statistical learning theory is their exponential concentration around the mean, expressed by the following tail bounds, which can be found, for example, in \cite[Proposition 2.5.2]{vershynin2018}.

\begin{proposition}[one-sided tail bounds for centered subgaussians]\label{prop:subgaussian-tails-onesided} $ $\\
    Let $X$ be subgaussian with variance proxy $\sigma^2 > 0$. Then it satisfies the tail bounds
    \begin{align*}
        \P(X - \E[X] \geq t) &\leq \exp\left(- \frac{t^2} {2\sigma^2}\right) \quad \forall t\geq 0\\
        \text{and} \quad \P(X - \E[X] \leq -t) &\leq \exp\left(- \frac{t^2} {2\sigma^2}\right) \quad \forall t\geq 0.
    \end{align*}
\end{proposition}

In Section~\ref{sec:rates}, we need to control both the means and deviations above the means of the two generalization errors. For the former problem, we note that $\egenmu$ and $\egenlambda$ are the suprema of the respective processes of absolute differences between empirical and actual loss terms. These processes turn out to be related to subgaussians via their increments.

\begin{definition}[subgaussian increments]\label{def:subgaussian-increments} $ $\\
    A real-valued stochastic process $(X_t)_{t\in T}$ indexed over a metric space $(T,\rho)$ is said to have \emph{subgaussian increments} if $\E[X_t] = 0$ and its increments $X_s-X_t$ are subgaussian with variance proxy $\rho(s,t)^2$ for all $s,t\in T$.
\end{definition}

Dudley's inequality bounds the mean of the supremum of such a subgaussian process in terms of the metric entropy of the index space, which requires the following definition.

\begin{definition}[$\epsilon$-nets; covering numbers]$ $\\
    Let $(T,\rho)$ be a metric space and $\epsilon>0$. A subset $N$ of $T$ is called an \emph{$\epsilon$-net} of $(T,\rho)$ if for every $t \in T$ there exists a point $\pi(t) \in N$ such that $\rho(t,\pi(t)) \leq \epsilon$. The \emph{covering numbers} of $T$ are defined as
    \begin{equation*}
        N(T,\rho,\epsilon) = \min \{\card{N} \mid N \ \text{is an $\epsilon$-net of} \ (T,\rho) \},
    \end{equation*}
    i.e., the minimal cardinality of an $\epsilon$-net of $(T,\rho)$.
\end{definition}

Covering numbers can be defined on any metric space, but compactness ensures that they are finite for all $\epsilon>0$, sparing us some undesirable special cases in all further considerations. It is a common mistake to assume that covering numbers are monotonic with respect to inclusion. They are not, but we do have the following inequality for subsets.

\begin{proposition}[covering numbers of subsets]\label{prop:covering-number-monotonicity} $ $\\
    For any metric space $(T,\rho)$, any nonempty subset $T'\subseteq T$ and $\epsilon>0$, we have
    \begin{equation*}
        N(T',\rho,\epsilon) \leq N(T,\rho,\tfrac{\epsilon}{2}).
    \end{equation*}
\end{proposition}

\begin{proof}
    Unlike covering numbers, the packing numbers $D(T,\rho,\epsilon)$ are monotonic with respect to inclusion -- any $\epsilon$-packing of $T'$ is also an $\epsilon$-packing of $T$. By \cite[Lemma~5.12]{vanhandel2014} we have
    \begin{equation*}
        N(T',\rho,\epsilon) \leq D(T',\rho,\epsilon) \leq D(T,\rho,\epsilon) \leq N(T,\rho,\tfrac{\epsilon}{2}). \tag*{\qedhere}
    \end{equation*}
\end{proof}

A proof of the following version of Dudley's inequality can be found in \cite[Corollary~5.25]{vanhandel2014}. Note that, since $(X_t)_{t\in T}$ is path-wise continuous and $(T,\rho)$ is compact, $\sup_{t\in T} X_t$ is measurable. This is a special case of the more general notion of a separable process in \cite{vanhandel2014}.

\begin{theorem}[Dudley's inequality, integral form]\label{thm:dudley-integral} $ $\\
    Let $(X_t)_{t\in T}$ be a path-wise continuous process with subgaussian increments on a compact metric space $(T,\rho)$. Then we have the following estimate:
    \begin{equation*}
        \E\left[ \sup_{t\in T} X_t \right] \leq 12 \int_0^\infty \sqrt{\log N(T,\rho,\epsilon)}\, d\epsilon.
    \end{equation*}
\end{theorem}

For the second issue, the deviation of the generalization errors above their means, note that $\egenmu(n)$ can be expressed as a function of the first $n$ real samples $Y_1,\dotsc,Y_n$ and $\egenlambda(n)$ as a function of the first $n$ noise samples $Z_1,\dotsc,Z_n$. In either case, the dependence on these inputs satisfies one of the following conditions.

\begin{definition}[separately bounded / Lipschitz observables] $ $\\
    Let $(M,\d)$ be a metric space, $n\in\N$ and $K: M^n \to \R$ measurable.
    \begin{itemize}
        \item We say that $K$ is \emph{separately bounded} with coefficients $c_1,\dotsc,c_n \geq 0$ if
        \begin{equation*}
            \abs{K(x_1,\dotsc,x_i,\dotsc,x_n) - K(x_1,\dotsc,x_i',\dotsc,x_n)} \leq c_i
        \end{equation*}
        holds for all $x_1,\dotsc,x_n,x_1',\dotsc,x_n'\in M$ and $i\in \otn$.
        \item We say that $K$ is \emph{separately Lipschitz} with coefficients $L_1,\dotsc,L_n \geq 0$ if
        \begin{equation*}
            \abs{K(x_1,\dotsc,x_i,\dotsc,x_n) - K(x_1,\dotsc,x_i',\dotsc,x_n)} \leq L_i \d(x_i,x_i')
        \end{equation*}
        holds for all $x_1,\dotsc,x_n,x_1',\dotsc,x_n'\in M$ and $i\in \otn$.
    \end{itemize}
\end{definition}

The following concentration result for separately bounded observables of independent random variables is originally due to \cite{mcdiarmid1989}, with a more modern examination found in \cite[Chapter~23]{douc2018}.

\begin{theorem}[McDiarmid's inequality]\label{thm:mcdiarmid-iid} $ $\\
    Let $(M,\d)$ be a metric space, $n\in\N$ and $K: M^n \to \R$ a separately bounded observable with coefficients $c_1,\dotsc,c_n\geq 0$. Moreover, let $X_1,\dotsc,X_n$ be independent random variables taking values in $M$. Then $K(X_1,\dotsc,X_n)$ is subgaussian with variance proxy
    \begin{equation*}
        \vpnorm{K(X_1,\dotsc,X_n)}^2 \leq \frac 14 \sum_{i=1}^n c_i^2.
    \end{equation*}
\end{theorem}

The above result arises naturally in the proof. An exponential tail-bound, which is the classical statement of McDiarmid's inequality, can be recovered by applying Proposition~\ref{prop:subgaussian-tails-onesided}.

I.i.d. processes are not the only ones that satisfy such an exponential concentration inequality. Similar results have been proven for other random processes, e.g. geometrically ergodic Markov chains \cite{dedecker2015,douc2018,havet2020}, as well as deterministic processes arising from a dynamical system in the sense of Remark~\ref{rem:observable-processes} \cite{collet2002,cuny2023}. We focus on one such result, for systems modeled by an aperiodic uniform Young tower with exponential tails, due to \cite[Theorem~7.1]{chazottes2012}. Once again, a version with tail bounds can be recovered by Proposition~\ref{prop:subgaussian-tails-onesided}.

\begin{theorem}[Chazottes-Gouëzel inequality]\label{thm:mcdiarmid-young} $ $\\
    Let $(M,T)$ be a dynamical system modeled by an aperiodic uniform Young tower with exponential tails, $\d$ the associated metric and $\mu$ its associated mixing $T$-invariant distribution by Proposition~\ref{prop:young-tower-invariant-measure}.
    
    Let $n\in\N$ and $K: M^n \to \R$ a separately Lipschitz observable with coefficients $L_1,\dotsc,L_n\geq 0$. Moreover, let $X$ be an $M$-valued random variable with $X\sim \mu$. Then $K(X,T(X),\dotsc,T^{n-1}(X))$ is subgaussian with variance proxy
    \begin{equation*}
        \Vpnorm{K\bigl(X,T(X),\dotsc,T^{n-1}(X)\bigr)}^2 \leq C \sum_{i=1}^n L_i^2,
    \end{equation*}
    where $C>0$ is a constant depending only on $(M,T)$.
\end{theorem}

\section{Hölder Spaces}\label{ap:hoelder}

In the following, we always assume $d,m\in \N$. We follow \cite{asatryan2023} for the definitions.

\begin{definition}[multi-indices, $C^k$-spaces]\label{def:C^k-space} $ $\\
    If $\beta = (\beta_1,\dotsc,\beta_d)\in \N_0^d$, we refer to $\beta$ as \emph{multi-index} and write $\card{\beta} = \sum_{i=1}^d \beta_i$.
    
    Let $U\subseteq \R^d$ be bounded, open, and convex. For $f: U\to \R^{m}$ we write $D_\beta f(x) = \frac{\partial^{\card{\beta}}}{\partial x_1^{\beta_1} \dots \partial x_d^{\beta_d}} f(x)$.
    For $k\in \N_0$ we define $C^k(U,\R^m)$ as the set of all $k$-times continuously differentiable functions $f : U \to \R^m$, i.e., $D_\beta f$ exists and is continuous for all $\beta \in \N_0^d$, $\card{\beta}\leq k$.
    
    Let $A \subseteq \R^d$ be compact and convex with nonempty interior $A^\circ$. For $k\in \N_0$ we define $C^k(A,\R^m)$ as the set of all continuous  functions $f : A \to \R^m$ such that $f|_{A^\circ} \in C^k(A^\circ,\R^m)$ and $D_\beta (f|_{A^\circ})$ has a continuous extension to $A$ for all $\beta \in \N_0^d$, $\card{\beta}\leq k$.

    It is well-known that $C^k(A,\R^m)$ is a Banach space with the norm
    \begin{equation*}
        \norm{f}_{C^k} = \max_{\card{\beta}\leq k} \sup_{x\in A} \abs{D_\beta f(x)}.
    \end{equation*}
    As usual, we write $C$ instead of $C^0$ and $\norm{\cdot}_\infty$ instead of $\norm{\cdot}_{C^0}$.
\end{definition}

\begin{definition}[Hölder spaces] $ $\\
    Let $A \subseteq \R^d$ be compact and convex with nonempty interior. We say that $f: A \to \R^m$ is \emph{$\alpha$-Hölder continuous} for the Hölder exponent $\alpha \in (0,1]$ if
    \begin{equation*}
        [f]_\alpha = \sup_{\substack{x,y\in A \\ x\neq y}} \frac{\abs{f(x)-f(y)}}{\abs{x-y}^\alpha}
    \end{equation*}
    is finite. In that case, we call $[f]_\alpha$ the \emph{$\alpha$-Hölder coefficient} for $f$.

    For $k\in \N_0$ and $\alpha \in (0,1]$ we define $C^{k,\alpha}(A,\R^m)$ as the set of all $f\in C^k(A,\R^m)$ such that $D_\beta f$ is $\alpha$-Hölder continuous for all $\beta\in \N_0^d$, $\abs{\beta}=k$.

    By \cite{gilbarg2001}, $C^{k,\alpha}(A,\R^m)$ is a Banach space with the norm
    \begin{equation*}
        \norm{f}_{C^{k,\alpha}} = \norm{f}_{C^k} + \max_{\card\beta = k} \left[D_\beta f\right]_\alpha.
    \end{equation*}
\end{definition}

To work with these spaces, we need the following easy to verify lemma.

\begin{lemma}[component-wise (Hölder) continuity]\label{lem:norm-subadditivity} $ $\\
    Let $A \subseteq \R^d$ be compact and convex with nonempty interior, $k\in \N_0$, $\alpha \in (0,1]$ and $f: A \to \R^m$ with $f = (f_1,\dotsc, f_m)$. Then the following equivalences and bounds hold:
    \begin{enumerate}[label=(\roman*)]
        \item $f\in C^k(A,\R^m)$ if and only if $f_i\in C^k(A,\R)$ for all $i\in \{1,\dotsc,m\}$. Then
        \begin{equation*}
            \norm{f_i}_{C^k} \leq \norm{f}_{C^k} \leq \sqrt{\sum_{j=1}^m \norm{f_j}_{C^k}^2}.
        \end{equation*}
        \item $f\in C^{k,\alpha}(A,\R^m)$ if and only if $f_i\in C^{k,\alpha}(A,\R)$ for all $i\in \{1,\dotsc,m\}$. Then
        \begin{equation*}
            \norm{f_i}_{C^{k,\alpha}} \leq \norm{f}_{C^{k,\alpha}} \leq \sum_{j=1}^m \norm{f_j}_{C^{k,\alpha}}.
        \end{equation*}
    \end{enumerate}
\end{lemma}

Since the Hölder spaces are infinite-dimensional, the unit ball $B_{C^{k,\alpha}(A,\R^m)}(0,1)$ is not compact and we cannot expect it to have finite covering numbers. However, it turns out that weakening the topology just a bit gives the desired compactness, as the following proposition shows.

\begin{proposition}[compact embeddings of Hölder balls]\label{prop:hoelder-embeddings}$ $\\
    Let $A \subseteq \R^d$ be compact and convex with nonempty interior, $k\in \N_0$, $\alpha \in (0,1]$ and consider the closed unit ball $B = B_{C^{k,\alpha}(A,\R^m)}(0,1)$ w.r.t. $\norm{\cdot}_{C^{k,\alpha}}$. Then $B$ is compact $\dots$
    \begin{enumerate}[label=(\roman*)]
        \item $\dots$ in the $\norm{\cdot}_{C^{k,\alpha'}}$-topology for $0<\alpha'<\alpha$.
        \item $\dots$ in the $\norm{\cdot}_{C^{k'}}$-topology for all $k'\in \N_0$, $k'\leq k$.
    \end{enumerate}
\end{proposition}

\begin{proof}
    The embedding $C^{k,\alpha} \hookrightarrow C^{k,\alpha'}$ is compact by \cite[Lemma~6.33]{gilbarg2001} (only for $m=1$, though the general case is analogous), so $B$ is precompact in the $\norm{\cdot}_{C^{k,\alpha'}}$-topology. Closure follows from the fact that the defining inequality $\norm{f}_{C^{k,\alpha}}\leq 1$ of $B$ is stable under $\norm{\cdot}_{C^k}$-convergence, thus proving assertion (i). For assertion (ii), note that compactness is maintained when passing to an even weaker topology.
\end{proof}

Compactness in a weaker $\norm{\cdot}_{C^{k,\alpha'}}$-topology is already sufficient for the existence of minimax solutions in Section~\ref{sec:infinite-dim-model}. If we weaken the topology further, we can even get explicit bounds on the metric entropy, which prove crucial for the quantitative results in Section~\ref{sec:rates}.

For these next two results we only consider $A=[0,1]^d$, because this keeps the notation and bounds simpler. While we choose the same simplification in Assumption~\ref{ass:general-setting}, it is important to reiterate that these results are not unique to $[0,1]^d$; similar results can be proven for general $A$, following the same arguments.

\begin{proposition}[metric entropy of $C^{k,\alpha}$ embedded into $C$]\label{prop:hoelder-entropy-sup} $ $\\
    For all $k\in \N_0$, $\alpha\in (0,1]$ and $\epsilon > 0$ we have
    \begin{equation*}
        \log N(B_{C^{k,\alpha}([0,1]^d,\R^m)}(0,1), \norm{\cdot}_\infty, \epsilon) \leq \gamma \epsilon^{-\frac{d}{\alpha+k}},
    \end{equation*}
    where $\gamma= \gamma(d,m,\alpha+k)$ is a constant depending only on $d$, $m$ and $\alpha+k$.
\end{proposition}

\begin{proof}
    For $m=1$, see \cite[Theorem~2.7.1]{vandervaart2023}. For $m>1$, see \cite[Theorem~4.2]{asatryan2023}.
\end{proof}

For the ergodic generalization error $\egenmu(n)$, we only obtain subgaussian increments with respect to the $C^1$-norm in Section~\ref{sec:rates}, which requires a different metric entropy estimate.

\begin{proposition}[metric entropy of $C^{k+1,\alpha}$ embedded into $C^1$]\label{prop:hoelder-entropy-C1} $ $\\
    For all $k\in \N_0$, $\alpha\in (0,1]$ and $\epsilon > 0$ we have
    \begin{equation*}
        \log N(B_{C^{k+1,\alpha}([0,1]^d,\R^m)}(0,1), \norm{\cdot}_{C^1}, \epsilon) \leq \hat\gamma \epsilon^{-\frac{d}{\alpha+k}},
    \end{equation*}
    where $\hat\gamma= \hat\gamma(d,m,\alpha+k)$ is a constant depending only on $d$, $m$ and $\alpha+k$.
\end{proposition}

\begin{proof}
    We prove this w.l.o.g. for $m=1$. The generalization to $m>1$ works exactly as in \cite[Theorem~4.2]{asatryan2023}. Let $B=B_{C^{k+1,\alpha}([0,1]^d,\R)}(0,1)$. Fix $\epsilon>0$ and let $J$ be an $\tfrac{\epsilon}{8d}$-net of $B_{C^{k,\alpha}([0,1]^d,\R)}(0,1)$ w.r.t. $\norm{\cdot}_\infty$, with minimal cardinality. Then by Proposition~\ref{prop:hoelder-entropy-sup},
    \begin{equation*}
        \log \card{J} \leq \gamma(d,1,\alpha+k) 2^{\frac{3d}{\alpha+k}} d^{\frac{d}{\alpha+k}} \epsilon^{-\frac{d}{\alpha+k}}.
    \end{equation*}
    Let $C$ be an $\frac{\epsilon}{4}$-net of $[-1,1]$ with minimal cardinality. It is easy to see that $\card{C}\leq \frac{4}{\epsilon}$.
    
    We consider the gradient $\nabla : C^1([0,1]^d,\R) \to C([0,1]^d,\R^d)$ as a linear operator.
    This operators is also continuous, as Lemma~\ref{lem:norm-subadditivity}~(i) implies
    \begin{equation*}
        \norm{\nabla F}_\infty \leq \sqrt{\sum_{i=1}^d \Norm{\frac{\partial}{\partial x_i} F}_\infty^2} \leq \sqrt{d} \cdot \norm{F}_{C^1}
    \end{equation*}
    for all $F \in C^1$. Since $B$ is compact in the $\norm{\cdot}_{C^1}$-topology by Proposition~\ref{prop:hoelder-embeddings}~(ii), its image $\nabla(B)$ is compact in the $\norm{\cdot}_\infty$-topology.
    Consequently, for any $g\in J^d$ the continuous function $\psi_g : \nabla(B) \to \R, \quad \hat g \mapsto \norm{g-\hat g}_\infty$ attains its minimum on $\nabla(B)$.
    We associate to each $g\in J^d$ a fixed $g^* \in \argmin_{\hat g\in \nabla(B)} \norm{g-\hat g}_\infty$.
    By definition of $\nabla(B)$, there exists some $G^*\in B$ with $\nabla G^* = g^*$. Now define
    \begin{equation*}
        J' = \{ G^* + c - G^*(0,\dotsc,0) \mid g\in J^m, c\in C\} \subseteq C^{k+1,\alpha}([0,1]^d,\R).
    \end{equation*}
    Note that, in shifting $G^*$ by a constant, we leave its gradient $g^*$ unchanged but may possibly move it outside the unit ball $B$. Thus, for any $g\in J^d$ and $c\in C$ there exists a $G\in J'$ with $\norm{g-\nabla G}_\infty = \min_{\hat g\in \nabla(B)} \norm{g-\hat g}_\infty$ and $G(0,\dotsc,0) = c$.
    
    We assert that $J'$ is an $\frac\epsilon 2$-net of $B\cup J'$ w.r.t. $\norm{\cdot}_{C^1}$. Indeed, let $F\in B$ and let $\nabla F=(f_1,\dotsc,f_d)$
    For all $i\in \{1,\dotsc,d\}$ we have $f_i \in C^{k,\alpha}([0,1]^d,\R)$ with $\norm{f_i}_{C^{k,\alpha}} \leq \norm{F}_{C^{k+1,\alpha}} \leq 1$, since the partial derivatives of $f_i$ up to the $k$-th degree are a subset of the partial derivatives of $F$ up to the $(k+1)$-th degree.
    Thus, by the choice of $J$, there exists $g=(g_1,\dotsc,g_d)\in J^d$ with $\norm{f_i-g_i}_\infty \leq \frac{\epsilon}{8d}$, $i\in \{1,\dotsc,d\}$.
    By Lemma~\ref{lem:norm-subadditivity}~(i), it follows that $\norm{\nabla F-g}_\infty \leq \frac{\epsilon}{8\sqrt{d}}$.
    Moreover, since $F\in B$, we have $\norm{F}_\infty\leq 1$ and thus $F(0,\dotsc,0)\in [-1,1]$. Therefore, we can choose $G\in J'$ with
    \begin{equation*}
        \norm{g-\nabla G}_\infty = \min_{\hat g\in \nabla(B)} \norm{g-\hat g}_\infty \quad \text{and} \quad \abs{F(0,\dotsc,0)-G(0,\dotsc,0)} \leq \frac{\epsilon}{4}.
    \end{equation*}
    Since $\nabla F\in \nabla(B)$, we can estimate
    \begin{equation*}
        \norm{\nabla F-\nabla G}_\infty \leq \norm{\nabla F-g}_\infty + \norm{g-\nabla G}_\infty \leq 2 \norm{\nabla F-g}_\infty \leq \frac{\epsilon}{4\sqrt{d}}.
    \end{equation*}
    By Lemma~\ref{lem:norm-subadditivity}~(i), it follows immediately that
    \begin{equation*}
        \Norm{\frac{\partial}{\partial x_i}(F-G)}_\infty \leq \frac{\epsilon}{4\sqrt{d}} \leq \frac \epsilon 2
    \end{equation*}
    holds for all $i\in\{1,\dotsc,d\}$. It remains to show that $\norm{F-G}_\infty \leq \frac\epsilon 2$. To this end, fix $x\in [0,1]^d\setminus\{(0,\dotsc,0)\}$ and define the differentiable curve $\varphi : \left[0, \abs{x}\right] \to [0,1]^d,\quad t \mapsto \frac{tx}{\abs{x}}$.
    We have $\abs{\varphi'(t)} =1$ constantly. Therefore, we have
    \begin{align*}
        \abs{F(x)-G(x)} &\leq \abs{F(0,\dotsc,0)-G(0,\dotsc,0)} + \abs{(F-G)(x)-(F-G)(0,\dotsc,0)}\\
        &\leq \frac{\epsilon}{4} + \abs{(F-G)(x)-(F-G)(0,\dotsc,0)}\\
        &= \frac{\epsilon}{4} + \Abs{\int_0^{\abs{x}} \bigl(\nabla F(\varphi(t)) - \nabla G(\varphi(t))\bigr)\cdot \varphi'(t) \, dt}\\
        &\leq \frac{\epsilon}{4} + \int_0^{\abs{x}} \babs{\nabla F(\varphi(t)) - \nabla G(\varphi(t))} \cdot \abs{\varphi'(t)} \, dt\\
        &\leq \frac{\epsilon}{4} + \sqrt{d} \bnorm{\nabla F-\nabla G}_\infty \leq \frac \epsilon 2,
    \end{align*}
    where we used the fundamental theorem of Calculus for line integrals in the third line, followed by the triangle inequality and finally the fact that $\diam([0,1]^d)= \sqrt{d}$. Altogether, we have $\norm{F-G}_{C^1} \leq \frac \epsilon 2$, and since $F \in B$ was arbitrary, $J'$ is indeed an $\frac \epsilon 2$-net of $B\cup J'$. By Proposition~\ref{prop:covering-number-monotonicity}, it follows that
    \begin{align*}
        \log N(B, \norm{\cdot}_{C^1}, \epsilon) &\leq \log N(B\cup J', \norm{\cdot}_{C^1}, \tfrac{\epsilon}{2})\\ &\leq \log \card{J'} \leq \log (\card C \cdot \card J^d)\\
        &\leq \underbrace{\gamma(d,1,\alpha+k) 2^{\frac{3d}{\alpha+k}} d^{1+\frac{d}{\alpha+k}}}_{=\ \widetilde\gamma\ =\ \widetilde \gamma(d,\alpha+k)} \epsilon^{-\frac{d}{\alpha+k}} + \log \frac{4}{\epsilon}.
    \end{align*}
    To finish the proof, we need to bound $\log\frac 4\epsilon$ by a multiple of $\epsilon^{-\frac{d}{\alpha+k}}$. Consider the function
    \begin{equation*}
        a: [0,\infty) \to \R, \quad \epsilon \mapsto a(\epsilon) = \begin{cases}
            \epsilon^{\frac{d}{\alpha+k}}\cdot \log\frac 4\epsilon & \text{if} \ \epsilon > 0,\\
            0 & \text{else.}
        \end{cases}
    \end{equation*}
    $a$ is continuous on the compact set $[0,4]$ (check using L'Hospital's rule at $0$) and thus attains a maximum $M=M(d,\alpha+k)\geq 0$ on this set. For all $\epsilon>4$, $a(\epsilon)$ is negative, so $M$ is a global Maximum and we can bound
    \begin{equation*}
        \log N(B, \norm{\cdot}_{C^1}, \epsilon) \leq (\widetilde \gamma + M)\epsilon^{-\frac{d}{\alpha+k}}
    \end{equation*}
    for all $\epsilon>0$, proving the assertion with $\hat\gamma(d,1,\alpha+k) = \widetilde \gamma(d,\alpha+k) + M(d,\alpha+k)$.
\end{proof}

\phantomsection
\addcontentsline{toc}{section}{References}
\printbibliography[title={References}]

\end{document}